\documentclass{article}
\usepackage{booktabs,multirow,xcolor,colortbl}
\newcommand{\e}[1]{\times 10^{#1}}

\usepackage{microtype}
\usepackage{graphicx}
\usepackage{subcaption}
\usepackage{booktabs} 
\usepackage{xurl}
\usepackage{mathrsfs}
\usepackage{float}
\usepackage{hyperref}
\usepackage{amsfonts}

\usepackage[preprint]{icml2026}

\usepackage{amsmath}
\usepackage{amssymb}
\usepackage{mathtools}
\usepackage{amsthm}

\usepackage{xurl}
\usepackage[capitalize,noabbrev]{cleveref}

\theoremstyle{plain}

\theoremstyle{definition}

\theoremstyle{remark}

\usepackage[textsize=tiny]{todonotes}
\usepackage[table]{xcolor}
\usepackage[most]{tcolorbox}
\usepackage[framemethod=TikZ]
{mdframed}
\usepackage{pifont} 
\usepackage{tikz}
\usepackage{tabularx}
\usepackage{array}
\newcolumntype{Y}{>{\raggedright\arraybackslash}X}
\newcolumntype{P}[1]{>{\raggedright\arraybackslash}p{#1}}
\usepackage{hyperref}
\hypersetup{
    colorlinks=true,
    linkcolor=orange,
    filecolor=magenta,      
    urlcolor=cyan,
    citecolor=blue,
    pdftitle={MAAT},
    pdfpagemode=FullScreen,
}

\newtheorem{Lemma}{Lemma}
\newtheorem{Definition}{Definition}
\newtheorem{Proposition}{Proposition}
\newtheorem{Corollary}{Corollary}
\definecolor{okgreen}{RGB}{0,128,0}
\definecolor{badred}{RGB}{180,0,0}
\definecolor{amber}{RGB}{184,134,11}

\newcommand{\cgood}{{\textcolor{okgreen}{\ding{51}}}}   
\newcommand{\cbad}{{\textcolor{badred}{\ding{55}}}}    
\newcommand{\cpart}{{\textcolor{amber}{\textasciitilde}}} 

\definecolor{contribgray}{gray}{0.85}
\colorlet{contribgraybg}{contribgray!35}
\colorlet{contribgrayframe}{gray!55}
\newenvironment{contribbox}[1][]{%
  \begin{tcolorbox}[
    enhanced,
    breakable,
    title={#1},
    colback=contribgraybg,
    colframe=contribgrayframe,
    colbacktitle=contribgraybg,
    fonttitle=\bfseries,
    coltitle=black,
    attach boxed title to top left={yshift=-3mm, xshift=2mm},
    boxed title style={size=small, colback=contribgraybg, frame hidden},
    sharp corners,
    rounded corners,
    arc=3mm,
    top=2mm,
    bottom=1mm,
    left=1mm,
    right=1mm,
    boxsep=0.8mm,
    width=\linewidth,
  ]%
}{%
  \end{tcolorbox}
}
\icmltitlerunning{Knowledge-Informed Kernel State Reconstruction }

\begin{document}

\twocolumn[
  \icmltitle{Knowledge-Informed Kernel State Reconstruction\\ from Heterogeneous Partial Observations}



  \icmlsetsymbol{equal}{*}

  \begin{icmlauthorlist}
    \icmlauthor{Luca Muscarnera}{equal,yyy}
    \icmlauthor{Silas {Ruhrberg Estévez}}{equal,yyy}
    \icmlauthor{Samuel Holt}{yyy}
    \icmlauthor{Evgeny Saveliev}{yyy}
    \icmlauthor{Mihaela {van der Schaar}}{yyy}
  \end{icmlauthorlist}

  \icmlaffiliation{yyy}{DAMTP, University of Cambridge, Cambridge, UK}

  \icmlcorrespondingauthor{Luca Muscarnera}{lm2152@cam.ac.uk}

  \icmlkeywords{Dynamical Systems, State Reconstruction, Symbolic Regression, Kernel Methods}

  \vskip 0.3in
]



\printAffiliationsAndNotice{\icmlEqualContribution}  

\begin{abstract}
Real-world scientific systems are rarely observed through complete, regularly sampled state trajectories. Instead, measurements are often partial, noisy, and heterogeneous, providing fragmented views of latent dynamical states. We introduce \texttt{MAAT} (Model Aware Approximation of Trajectories), a framework for knowledge-informed Kernel State Reconstruction in partially observed dynamical systems. \texttt{MAAT} formulates reconstruction in a reproducing kernel Hilbert space and incorporates heterogeneous observation operators together with semantic and structural priors, including non-negativity, conservation constraints, and domain-specific measurement models. This yields smooth, physically consistent state estimates with analytic time derivatives, providing a principled interface between fragmented measurements and downstream mechanistic discovery methods such as symbolic regression. Across nine scientific benchmarks, multiple noise regimes, and a real-world COVID-19 dataset, \texttt{MAAT} substantially reduces trajectory and derivative reconstruction error relative to strong baselines.
\end{abstract}

\section{Introduction}
In many scientific and clinical domains, the variables are not directly observable \citep{rubanova2019latent}. Instead, systems are measured through heterogeneous data sources that provide only partial and indirect information about the underlying dynamics. For example, in oncology, tumour burden or clonal composition cannot be directly or continuously observed \citep{Bartolomucci2025}. Clinicians instead rely on sparse, high-specificity measurements such as imaging or genomic assays, alongside dense but indirect biomarkers derived from blood panels or physiological signals \citep{Sivapalan2023}. These observations differ in temporal resolution, noise characteristics, and their relationship to the latent state, making it non-trivial to reconstruct a coherent trajectory of disease progression.

This setting induces a fundamental data fusion problem:
recovering physically meaningful latent trajectories from
fragmented observations with heterogeneous structure. Mechanistic models expressed as differential equations provide a principled framework for reasoning about such systems \citep{strogatz2018nonlinear, Huang2013, Yu2024}. However, their application in practice is limited by the lack of reliable state trajectories and derivatives.

Existing approaches treat this reconstruction problem only partially. Classical smoothing methods produce continuous trajectories but ignore domain constraints and heterogeneous observation operators \citep{Rasmussen2005}. State-space models enable sensor fusion but require specifying transition dynamics a priori \citep{Kalman1960}. More flexible latent models accommodate partial observations but rely on black-box representations that obscure mechanistic structure \citep{Chen2018neuralode}. As a result, state reconstruction is typically treated as a preprocessing step, rather than as a \emph{representational bottleneck} that determines which downstream analyses are possible. We address this limitation by formulating state reconstruction as a knowledge-informed inference problem in function space, rather than as a purely numerical preprocessing step.

\begin{contribbox}[Contributions]
\textbf{Conceptual.}
We identify knowledge-informed state reconstruction from heterogeneous partial observations as a central bottleneck for mechanistic modeling.

\textbf{Technical.}
We introduce \texttt{MAAT}, a kernel-based framework
that embeds latent trajectories in a reproducing kernel
Hilbert space and incorporates heterogeneous observation
operators together with semantic and structural priors
directly into the reconstruction objective.

\textbf{Empirical.}
Across nine benchmark dynamical systems and a real-world COVID-19 dataset, \texttt{MAAT} improves trajectory and derivative reconstruction under multiple noise regimes relative to strong baselines.
\end{contribbox}
\section{Problem Formulation}
We consider a time-dependent dynamical system characterized by the state variable $x(t) \in \mathbb{R}^d$, governed by a system of first-order ordinary differential equations (ODEs):
\begin{equation}
    \dot{x}(t) = f(x(t)), \quad x(0) = x_0,
    \label{eq:dynamics}
\end{equation}
where $f: \mathbb{R}^d \to \mathbb{R}^d$ is an unknown vector field. In real-world settings, the time derivative $\dot{x}(t)$ is not directly observable. Measurements instead provide partial and noisy observations of the latent state $x(t)$, yielding a dataset $\mathcal{D} = \{(t_i, y_i, \mathcal{H}_i)\}_{i=1}^{N}$ of $N$ observations collected at irregular timestamps $t_i$. Each observation $y_i$ is related to the latent state $x(t_i)$ through a linear observation operator $\mathcal{H}_i$:
\begin{equation}
    y_i = \mathcal{H}_i\!\left(x(t_i)\right) + \epsilon_i, 
    \qquad \epsilon_i \sim \mathcal{N}(0, \Sigma),
    \label{eq:observation}
\end{equation}
where $\epsilon_i$ denotes measurement noise. The generalized observation model captures a core trade-off in experimental design: constraints on sensing, cost, and invasiveness necessitate balancing temporal resolution against state specificity. Consequently, datasets often combine sparse direct measurements (e.g., gene expression) with dense aggregated signals (e.g., blood biomarkers). Moreover, available domain knowledge can be leveraged not only via statistical regularization, but by enforcing interpretable, mechanistic constraints to guide state reconstruction.

\begin{figure}
    \centering
    \includegraphics[width=1.0\linewidth]{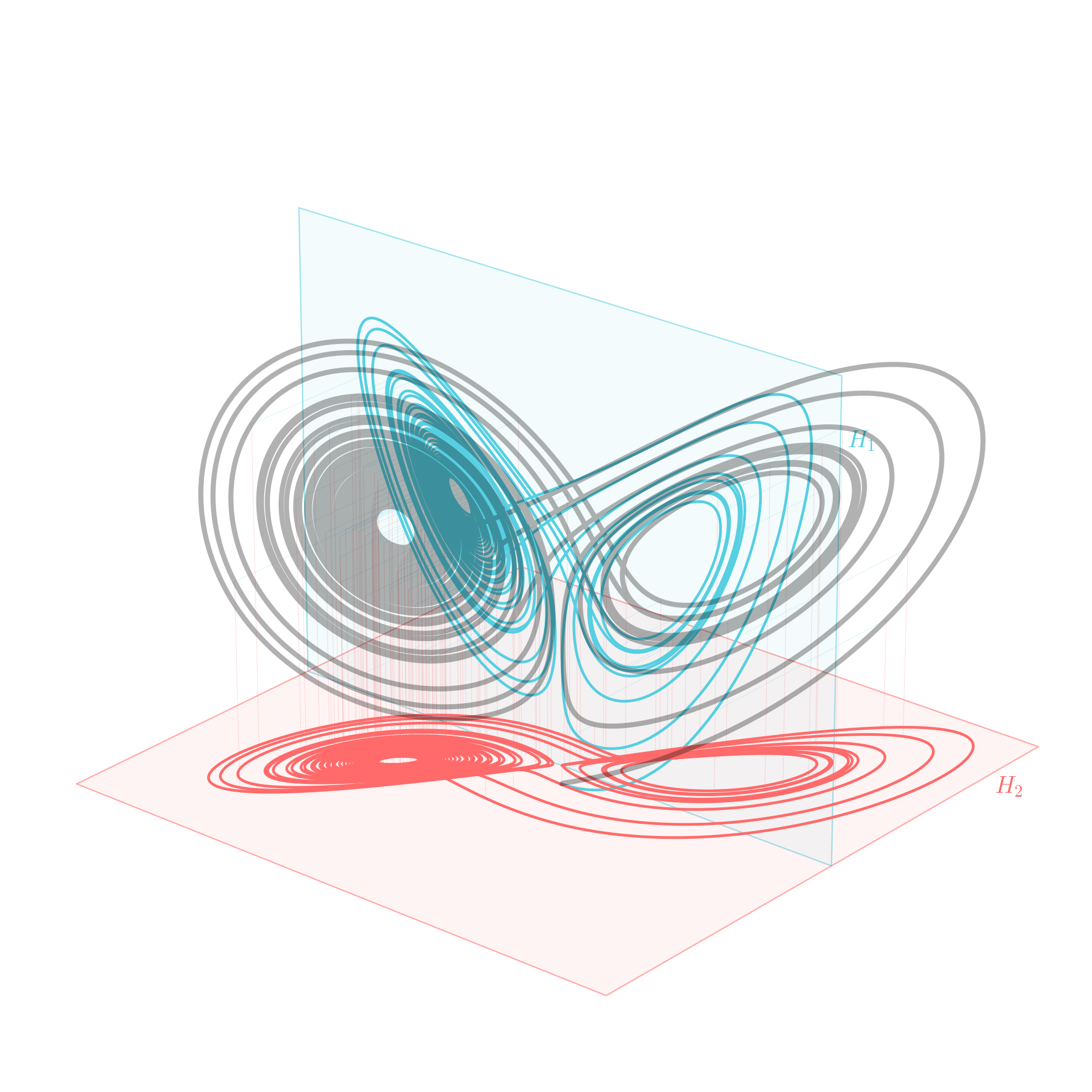}
    \caption{Observation operators represent different projections of the true underlying dynamics, that are observable through dense sampling}
    \label{fig:placeholder}
\end{figure}

\textbf{Optimization Framework.} Standard smoothing techniques such as splines fail to account for the structural constraints $\mathcal{C}$ or the heterogeneous operators $\mathcal{H}_i$. Instead, we formulate the reconstruction as a risk minimization problem in a Reproducing Kernel Hilbert Space (RKHS), denoted by $\mathcal{H}_K$. We seek the trajectory function $\hat{x} \in \mathcal{H}_K$ that minimizes the regularized empirical risk including domain knowledge of the form:
\vspace{-5pt}
\begin{equation}
\begin{split}
    \hat{x} & = \operatorname*{argmin}_{x \in \mathcal{H}_K} 
    \underbrace{\sum_{i=1}^N \| y_i - \mathcal{H}_i(x(t_i)) \|^2}_{\text{Data Fidelity}} + 
    \\
    & +
    \underbrace{\lambda_1 \|x\|_{\mathcal{H}_K}^2}_{\text{Smoothness}} + 
    \underbrace{\lambda_2 \mathcal{R}_{\text{phys}}(x, \mathcal{C})}_{\text{Physical Priors}},
    \label{eq:objective}
\end{split}
\end{equation}
\vspace{-15pt}

where $\|\cdot\|_{\mathcal{H}_K}$ denotes the RKHS norm, and $\mathcal{R}_{\text{phys}}$ is a penalty term enforcing domain knowledge (defined in Section 3). This formulation is equivalent to Maximum A Posteriori (MAP) estimation under Gaussian process priors, providing a principled framework to reconstruct derivatives.

\section{Knowledge-informed Kernel Regression}
We introduce \texttt{MAAT} (Model Aware Approximation of Trajectories), a framework for knowledge informed kernel regression for state estimation in dynamical systems. Unlike prior methods, \texttt{MAAT} goes beyond simple interpolation of the measurements but does the regression under the constraints provided by the physical knowledge about the variables in the systems and the interactions of the sub-components. A schematic overview of the framework is provided in Figure~\ref{fig:overview}:

\textbf{Kernel State Reconstruction (KSR).} We model the trajectory of each state variable $j \in \{1, \dots, d\}$ as a function in a Reproducing Kernel Hilbert Space (RKHS), $\widehat{x}_j(t) = \sum_{\ell=1}^N u_{\ell j} \kappa(t, t_\ell)$, where $\kappa$ is a smooth kernel (e.g., Gaussian) and $U \in \mathbb{R}^{N \times d}$ is a matrix of coefficients to be learned. This is a form of kernel ridge regression, and follows from the Representer Theorem. The coefficients $U$ are found by minimizing a composite loss function that balances fidelity to both snapshots and linear signals, along with regularization terms: 
\begin{equation}
\label{eq:ksr_loss}
\begin{split}
\min_{U} & \frac{w_s}{N_{\text{obs}}}\|\! \ \mathbf K^{\mathrm{obs}} \mathbf U - \mathbf X^{\mathrm{obs}}\! \ \|_F^2 + \sum_i \frac{w_i}{N}\|\! \ \mathbf K \mathbf U \mathbf H_i^\top - \mathbf Y\!\|_F^2 \\
& + \gamma \| \dot{ \mathbf K} \mathbf U - F(\mathbf K \mathbf U)\|_F^2 + \lambda \| \mathbf U\|_F^2,
\end{split}
\end{equation}
where  $K^{\mathrm{obs}}_{k\ell} = \kappa(t_k^{\text{obs}}, t_\ell)$ and  $K_{i\ell} = \kappa(t_i, t_\ell)$ are, respectively, the matrix of the kernel-based inner products between the timesteps where full state is observed and where the signal is observed and the matrix of the kernel-based inner products between the timesteps in which the fine grained signals are observed. The additional penalty from deviation from a prior ($\gamma \| \dot{ \mathbf K} \mathbf U - F(\mathbf K \mathbf U)\|_F^2 $) models the deviation from prior hypotheses on the dynamics of the system.

The object $\dot{\mathbf K}$ denotes the time derivative of the kernel Gram matrix $\mathbf K$. To see how this arises, note that the time derivative of the reconstructed trajectory at a new time point $t$ is given by
\vspace{-5pt}
\begin{align*}
    \partial_t \hat{\boldsymbol{x}}(t)
    = \partial_t \sum_j \boldsymbol e_j \kappa(t, t_j)
    = \mathbf U \, \partial_t \kappa(t, \boldsymbol t),
\end{align*}
\vspace{-20pt}

where the equality follows from the linearity of the model. Thus, the derivative operator acts only on the kernel. Since commonly used kernels (e.g., Gaussian kernels) are $\mathcal{C}^\infty$, their derivatives admit closed-form expressions and can be computed efficiently.

\begin{figure}
    \centering
    \includegraphics[width=\linewidth,trim={1cm 15cm 5cm 3cm}, clip]{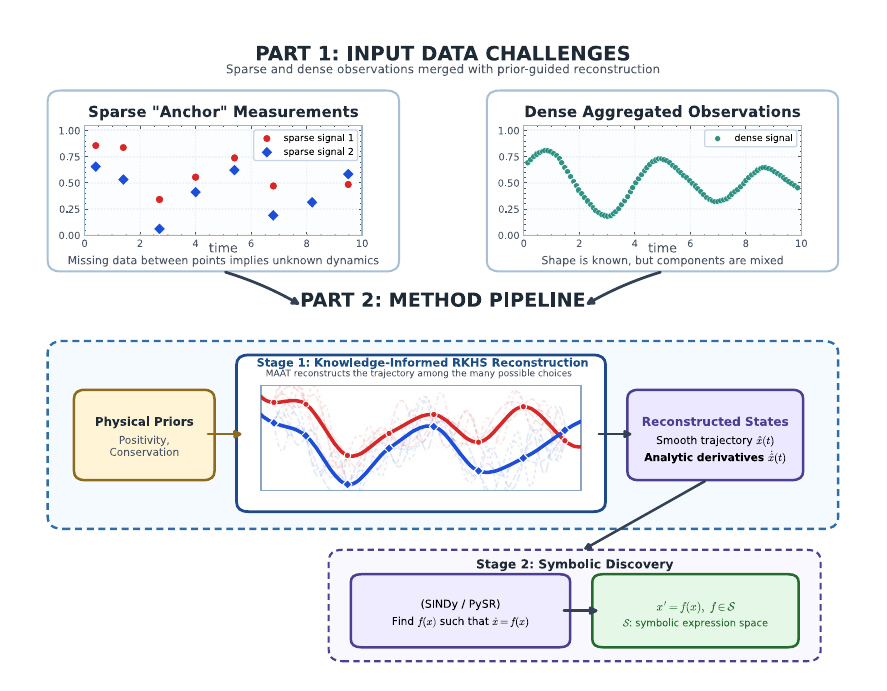}
    \vspace{-15pt}
    \caption{\textbf{Overview of \texttt{MAAT.}} Sparse anchor measurements and dense aggregate observations are combined with physical priors to produce reconstructed states through knowledge-informed kernel regression yielding smooth trajectories and analytical derivatives that can be used for symbolic regression.}
    \vspace{-15pt}
    \label{fig:overview}
\end{figure}

\textbf{Theoretical Justification.} This KSR approach is motivated by two key theoretical properties. First, the composite loss function is a valid surrogate for the true $L^2$ reconstruction error.

\begin{Lemma}[Composite loss is a calibrated surrogate]
Let $H:\mathbb{R}^d\!\to\!\mathbb{R}^p$ be a bounded linear observation operator. For any candidate trajectory $\hat{x}\in L^2([0,T];\mathbb{R}^d)$ and true trajectory $x$, define the risk $\mathcal{R}(\hat{x})=\|x-\hat{x}\|_{L^2}^2+\|H(x-\hat{x})\|_{L^2}^2$. Then
\[
\|x-\hat{x}\|_{L^2}^2 \;\le\; \mathcal{R}(\hat{x}) \;\le\; (1+\|H\|^2)\,\|x-\hat{x}\|_{L^2}^2.
\]
Hence minimizing $\mathcal{R}$ is equivalent to minimizing the $L^2$ reconstruction error up to a constant factor.
\end{Lemma}
\begin{proof}[Sketch]
The upper bound follows from the triangle inequality and the definition of the operator norm, $\|H(x-\hat{x})\|_{L^2}^2 \le \|H\|^2\|x-\hat{x}\|_{L^2}^2$. The lower bound is immediate. A full proof is in the Appendix B.
\end{proof}

Second, KSR provides derivative estimates that are fundamentally more robust to noise than standard numerical differentiation.

 \begin{Proposition}[FD noise floor vs KSR]
Assume additive i.i.d. noise with variance $\sigma^2$ on $x(t)$ sampled with step $\Delta t$. Central differences yield derivative error $\mathbb{E}\!\left[\|\widehat{\dot{x}}_{\mathrm{FD}}-\dot{x}\|_2^2\right] = \mathcal{O}(\Delta t^4)+\Omega(\sigma^2/\Delta t^2)$, which has an irreducible error floor. For KSR with regularization $\lambda$, the analytic derivative estimator satisfies $\mathbb{E}\!\left[\|\widehat{\dot{x}}_{\mathrm{KSR}}-\dot{x}\|_2^2\right] = \mathcal{O}(\lambda) + \mathcal{O}(\sigma^2/n)$. Thus, unlike finite differences, KSR avoids high-frequency noise amplification and admits a standard bias–variance trade-off. A proof sketch is provided in the Appendix \ref{app:theory}. This robustness is critical, as accurate derivatives are the most important input for successful symbolic regression \citep{Brunton2016}.
\end{Proposition}

\begin{table*}[t]
\centering
\caption{State reconstruction MSE ($\downarrow$) semi-synthetic benchmark datasets. Values are mean $\pm$ confidence interval. Best result for each dataset-backend pair is bolded. Noise type: Gaussian.}
\vspace{-5pt}
\label{tab:maat_selected_dataset_results}
\tiny
\setlength{\tabcolsep}{2.0pt}
\renewcommand{\arraystretch}{0.92}
\resizebox{\textwidth}{!}{%
\begin{tabular}{llccccccccc}
\toprule
\multirow{2}{*}{Method} & \multirow{2}{*}{Backend} & \multicolumn{3}{c}{Dynamical systems} & \multicolumn{3}{c}{Epidemiology / dynamics} & \multicolumn{3}{c}{Oncology / viral} \\
\cmidrule(lr){3-5}\cmidrule(lr){6-8}\cmidrule(lr){9-11}
& & CRC & Cons. & Neut. & SEIR & SEIRH & TMDD & Tumor & TDI & Viral \\
\midrule
\multirow{2}{*}{RBF}
& PySR & $1.5\e{-1}\!\pm\!1.8\e{-2}$ & $1.2\e{2}\!\pm\!1.7\e{1}$ & $2.9\e{-2}\!\pm\!2.4\e{-2}$ & $1.1\e{-2}\!\pm\!9.2\e{-4}$ & $9.0\e{-3}\!\pm\!7.9\e{-4}$ & $1.4\e{0}\!\pm\!1.3\e{-1}$ & $1.2\e{1}\!\pm\!1.2\e{0}$ & $5.6\e{1}\!\pm\!1.4\e{1}$ & $9.2\e{-3}\!\pm\!1.4\e{-3}$ \\
& SINDy & $1.5\e{-1}\!\pm\!1.8\e{-2}$ & $1.2\e{2}\!\pm\!1.7\e{1}$ & $2.9\e{-2}\!\pm\!2.3\e{-2}$ & $1.1\e{-2}\!\pm\!9.2\e{-4}$ & $9.0\e{-3}\!\pm\!7.9\e{-4}$ & $1.4\e{0}\!\pm\!1.3\e{-1}$ & $1.2\e{1}\!\pm\!1.3\e{0}$ & $5.7\e{1}\!\pm\!1.5\e{1}$ & $9.3\e{-3}\!\pm\!1.4\e{-3}$ \\
\addlinespace[1.2pt]
\multirow{2}{*}{Cubic}
& PySR & $2.3\e{-1}\!\pm\!4.1\e{-2}$ & $2.1\e{2}\!\pm\!4.7\e{1}$ & $4.7\e{-2}\!\pm\!3.2\e{-2}$ & $1.8\e{-2}\!\pm\!1.9\e{-3}$ & $1.5\e{-2}\!\pm\!2.9\e{-3}$ & $2.2\e{0}\!\pm\!2.3\e{-1}$ & $2.0\e{1}\!\pm\!3.2\e{0}$ & $1.6\e{2}\!\pm\!4.1\e{1}$ & $1.5\e{-2}\!\pm\!2.6\e{-3}$ \\
& SINDy & $2.3\e{-1}\!\pm\!4.1\e{-2}$ & $2.1\e{2}\!\pm\!4.5\e{1}$ & $4.7\e{-2}\!\pm\!3.2\e{-2}$ & $1.8\e{-2}\!\pm\!2.0\e{-3}$ & $1.5\e{-2}\!\pm\!2.9\e{-3}$ & $2.2\e{0}\!\pm\!2.3\e{-1}$ & $1.9\e{1}\!\pm\!2.7\e{0}$ & $1.6\e{2}\!\pm\!4.0\e{1}$ & $1.5\e{-2}\!\pm\!2.5\e{-3}$ \\
\addlinespace[1.2pt]
\multirow{2}{*}{GP}
& PySR & $3.7\e{-1}\!\pm\!3.3\e{-1}$ & $1.8\e{2}\!\pm\!1.9\e{2}$ & $7.5\e{-2}\!\pm\!6.8\e{-2}$ & $5.9\e{-3}\!\pm\!6.9\e{-3}$ & $2.9\e{-3}\!\pm\!2.1\e{-3}$ & $\mathbf{8.4\e{-2}\!\pm\!3.5\e{-2}}$ & $1.2\e{1}\!\pm\!1.4\e{1}$ & $2.9\e{2}\!\pm\!2.8\e{2}$ & $3.1\e{-2}\!\pm\!4.8\e{-2}$ \\
& SINDy & $3.1\e{-1}\!\pm\!2.7\e{-1}$ & $1.6\e{2}\!\pm\!1.8\e{2}$ & $1.7\e{-1}\!\pm\!4.1\e{-1}$ & $2.6\e{-3}\!\pm\!2.2\e{-3}$ & $8.6\e{-4}\!\pm\!4.2\e{-4}$ & $8.7\e{-2}\!\pm\!2.9\e{-2}$ & $1.3\e{1}\!\pm\!1.7\e{1}$ & $3.7\e{2}\!\pm\!3.4\e{2}$ & $1.3\e{-2}\!\pm\!1.2\e{-2}$ \\
\addlinespace[1.2pt]
\multirow{2}{*}{Kalman}
& PySR & $1.1\e{-2}\!\pm\!1.4\e{-3}$ & $\mathbf{9.0\e{0}\!\pm\!1.5\e{0}}$ & $2.5\e{-3}\!\pm\!2.5\e{-3}$ & $7.8\e{-4}\!\pm\!8.6\e{-5}$ & $7.7\e{-4}\!\pm\!9.5\e{-5}$ & $1.0\e{-1}\!\pm\!1.6\e{-2}$ & $8.7\e{-1}\!\pm\!1.3\e{-1}$ & $5.1\e{1}\!\pm\!1.4\e{1}$ & $7.1\e{-4}\!\pm\!1.4\e{-4}$ \\
& SINDy & $1.1\e{-2}\!\pm\!1.4\e{-3}$ & $1.3\e{1}\!\pm\!1.9\e{0}$ & $2.5\e{-3}\!\pm\!2.3\e{-3}$ & $8.4\e{-4}\!\pm\!8.3\e{-5}$ & $7.8\e{-4}\!\pm\!9.6\e{-5}$ & $1.0\e{-1}\!\pm\!1.5\e{-2}$ & $8.8\e{-1}\!\pm\!1.3\e{-1}$ & $4.7\e{1}\!\pm\!1.2\e{1}$ & $8.1\e{-4}\!\pm\!1.3\e{-4}$ \\
\addlinespace[1.2pt]
\multirow{2}{*}{Linear}
& PySR & $7.6\e{-2}\!\pm\!9.0\e{-3}$ & $6.3\e{1}\!\pm\!9.2\e{0}$ & $1.5\e{-2}\!\pm\!1.3\e{-2}$ & $5.7\e{-3}\!\pm\!5.4\e{-4}$ & $4.7\e{-3}\!\pm\!4.5\e{-4}$ & $7.0\e{-1}\!\pm\!6.1\e{-2}$ & $6.0\e{0}\!\pm\!7.7\e{-1}$ & $5.1\e{1}\!\pm\!1.2\e{1}$ & $4.7\e{-3}\!\pm\!8.0\e{-4}$ \\
& SINDy & $7.7\e{-2}\!\pm\!8.9\e{-3}$ & $6.4\e{1}\!\pm\!9.4\e{0}$ & $1.5\e{-2}\!\pm\!1.3\e{-2}$ & $5.7\e{-3}\!\pm\!5.5\e{-4}$ & $4.8\e{-3}\!\pm\!4.5\e{-4}$ & $7.0\e{-1}\!\pm\!6.2\e{-2}$ & $6.0\e{0}\!\pm\!7.4\e{-1}$ & $5.2\e{1}\!\pm\!1.3\e{1}$ & $4.8\e{-3}\!\pm\!7.9\e{-4}$ \\
\addlinespace[1.2pt]
\multirow{2}{*}{NeuralODE}
& PySR & $5.3\e{1}\!\pm\!9.9\e{1}$ & $6.3\e{11}\!\pm\!1.4\e{12}$ & $5.9\e{1}\!\pm\!7.8\e{1}$ & $9.3\e{-1}\!\pm\!6.1\e{-1}$ & $4.2\e{-1}\!\pm\!1.4\e{-1}$ & $1.8\e{0}\!\pm\!1.7\e{0}$ & $1.3\e{10}\!\pm\!2.9\e{10}$ & $5.1\e{3}\!\pm\!7.1\e{3}$ & $1.5\e{0}\!\pm\!1.2\e{0}$ \\
& SINDy & $3.0\e{0}\!\pm\!3.4\e{0}$ & $1.6\e{2}\!\pm\!2.3\e{2}$ & $2.0\e{0}\!\pm\!3.8\e{0}$ & $5.8\e{-1}\!\pm\!3.1\e{-1}$ & $3.3\e{-1}\!\pm\!1.3\e{-1}$ & $2.2\e{0}\!\pm\!2.2\e{0}$ & $1.7\e{2}\!\pm\!1.7\e{2}$ & $6.3\e{1}\!\pm\!5.8\e{1}$ & $9.1\e{-1}\!\pm\!7.8\e{-1}$ \\
\midrule
\multirow{2}{*}{\textbf{\texttt{MAAT}}}
& \multicolumn{1}{>{\columncolor{gray!7}}c}{PySR}
& \multicolumn{1}{>{\columncolor{gray!7}}c}{$\mathbf{4.0\e{-3}\!\pm\!1.7\e{-3}}$}
& \multicolumn{1}{>{\columncolor{gray!7}}c}{$2.0\e{1}\!\pm\!1.6\e{1}$}
& \multicolumn{1}{>{\columncolor{gray!7}}c}{$\mathbf{3.4\e{-4}\!\pm\!4.6\e{-4}}$}
& \multicolumn{1}{>{\columncolor{gray!7}}c}{$\mathbf{2.6\e{-5}\!\pm\!5.0\e{-6}}$}
& \multicolumn{1}{>{\columncolor{gray!7}}c}{$\mathbf{1.7\e{-5}\!\pm\!1.5\e{-6}}$}
& \multicolumn{1}{>{\columncolor{gray!7}}c}{$1.3\e{-1}\!\pm\!2.0\e{-1}$}
& \multicolumn{1}{>{\columncolor{gray!7}}c}{$\mathbf{3.0\e{-1}\!\pm\!1.9\e{-1}}$}
& \multicolumn{1}{>{\columncolor{gray!7}}c}{$\mathbf{4.9\e{0}\!\pm\!6.5\e{0}}$}
& \multicolumn{1}{>{\columncolor{gray!7}}c}{$\mathbf{4.7\e{-5}\!\pm\!2.9\e{-5}}$} \\
& \multicolumn{1}{>{\columncolor{gray!7}}c}{SINDy}
& \multicolumn{1}{>{\columncolor{gray!7}}c}{$\mathbf{1.5\e{-3}\!\pm\!1.6\e{-4}}$}
& \multicolumn{1}{>{\columncolor{gray!7}}c}{$\mathbf{5.8\e{0}\!\pm\!3.1\e{0}}$}
& \multicolumn{1}{>{\columncolor{gray!7}}c}{$\mathbf{4.3\e{-4}\!\pm\!5.2\e{-4}}$}
& \multicolumn{1}{>{\columncolor{gray!7}}c}{$\mathbf{7.9\e{-5}\!\pm\!1.1\e{-5}}$}
& \multicolumn{1}{>{\columncolor{gray!7}}c}{$\mathbf{4.1\e{-5}\!\pm\!5.0\e{-6}}$}
& \multicolumn{1}{>{\columncolor{gray!7}}c}{$\mathbf{4.8\e{-3}\!\pm\!4.2\e{-4}}$}
& \multicolumn{1}{>{\columncolor{gray!7}}c}{$\mathbf{1.2\e{-1}\!\pm\!4.2\e{-2}}$}
& \multicolumn{1}{>{\columncolor{gray!7}}c}{$\mathbf{1.8\e{0}\!\pm\!4.7\e{-1}}$}
& \multicolumn{1}{>{\columncolor{gray!7}}c}{$\mathbf{1.3\e{-4}\!\pm\!2.6\e{-5}}$} \\
\bottomrule
\end{tabular}}
\vspace{-10pt}
\end{table*}
\section{RELATED WORK}
We position \texttt{MAAT} within the landscape of derivative estimation and knowledge-informed machine learning using the structural criteria in Table~\ref{tab:estimators-capabilities}. A broader comparison with prior work is deferred to Appendix~\ref{app:extended_results}.

\textbf{Numerical and Smoothing Baselines.} Classical derivative estimation forms the foundation of most symbolic regression pipelines. While \emph{Finite differences} are computationally efficient, they amplify noise in low-SNR regimes. Windowed methods like \emph{Savitzky–Golay} \citep{Steinier1972} and variational approaches like \emph{TVRegDiff} \citep{Chartrand2011} improve robustness but provide discrete numerical outputs rather than analytic forms. While \emph{Cubic Splines} \citep{deBoor1978} and \emph{RBF Smoothing} \citep{Buhmann2003} provide differentiable surrogates, they typically only operate on single-channel, regularly sampled data. 

\textbf{Probabilistic and Filtering Estimators.} \emph{Gaussian Processes (GPs)} \cite{Rasmussen2005} and \emph{Kalman Filters} \cite{Kalman1960} offer a principled treatment of uncertainty and irregular sampling. GPs provide analytic derivatives through kernel differentiation; however, standard kernels struggle to scale to high-dimensional dynamical systems and do not natively support structural physical priors like mass conservation across state transitions. 

\textbf{Deep Latent Dynamics.}
Modern deep learning approaches, such as \emph{Neural ODEs} \citep{Chen2018neuralode}, utilize neural vector fields to represent latent dynamics.
\emph{Universal Differential Equations} (UDEs) \citep{Rackauckas2020} embed neural networks within mechanistic scaffolds, improving inductive bias and data efficiency, but they do not address heterogeneous observation operators.

\textbf{Comparison with Physics-Informed Kernel Learning}
A distinct but complementary line of research is the recently proposed \textit{Physics-Informed Kernel Learning} (PIKL) framework \citep{doumeche2025physics}. 
Similar to \texttt{MAAT}, PIKL reformulates the learning problem in a Reproducing Kernel Hilbert Space (RKHS) to overcome the training instabilities and lack of theoretical guarantees associated with Physics-Informed Neural Networks (PINNs). 
However, the two methods target fundamentally different objectives. 
PIKL addresses the \textit{forward} or \textit{hybrid modeling} problem: it assumes the differential operator $\mathcal{D}$ is \textbf{known} a priori (e.g., the wave or heat equation) and utilizes the kernel formulation to minimize a physics-informed risk $\mathcal{J}(f) = \|Y - f(X)\|^2 + \lambda \|\mathcal{D}f\|^2$, efficiently approximating the kernel via Fourier features to solve the PDE.

\section{EXPERIMENTS}
We evaluate \texttt{MAAT} on a series of diverse benchmarks to assess its performance in state reconstruction and downstream equation discovery. To evaluate the effectiveness of our state reconstruction, we assess the derived trajectories for downstream symbolic regression (SR) using two of the most popular SR algorithms: SINDy \cite{Brunton2016} and PySR \cite{Cranmer2023}.

\textbf{Experimental Setup.}
We evaluate our method on a total of ten datasets spanning diverse domains. Detailed descriptions of the underlying dynamical systems and data-generation procedures are provided in Appendix \ref{app:experimental_details}. For baseline comparisons in the state-reconstruction stage, we consider a broad set of standard smoothing and inference methods. Each method is used to reconstruct state trajectories and derivatives, which are then provided as input to the same symbolic regression (SR) engines for fair comparison.

\textbf{Penalizing Derivative Magnitude.} To test the robustness of our approach against noise, we analyze the impact of penalizing the magnitude of the derivative during reconstruction. 
We observed that in high-noise regimes, standard smoothing techniques often preserve high-frequency artifacts, leading to erroneous derivative estimates. By adding a penalty term to the derivative magnitude, \texttt{MAAT} recovers smoother, more physically plausible trajectories than baselines (see Table \ref{tab:maat_selected_dataset_results}). Additional results on correlated Gaussian and Student T noise are provided in Appendix \ref{app:extended_results}.

\begin{table}[!htb]
\centering
\caption{Effect of structural priors on \texttt{MAAT}. Values are state reconstruction MSE ($\downarrow$), reported as mean $\pm$ confidence interval. Best result per setting is bolded.}
\vspace{-5pt}
\label{tab:maat_priors_seir}
\scriptsize
\setlength{\tabcolsep}{3.2pt}
\renewcommand{\arraystretch}{0.92}
\resizebox{0.48\textwidth}{!}{%
\begin{tabular}{llcccc}
\toprule
\multirow{2}{*}{Noise} & \multirow{2}{*}{Dataset}
& \multicolumn{2}{c}{PySR} & \multicolumn{2}{c}{SINDy} \\
\cmidrule(lr){3-4}\cmidrule(lr){5-6}
& & Plain & +priors & Plain & +priors \\
\midrule
\multirow{2}{*}{Corr. Gauss.}
& SEIR  & $2.39\e{-5}\!\pm\!3.96\e{-6}$ & $\mathbf{2.00\e{-5}\!\pm\!1.80\e{-6}}$
        & $7.87\e{-5}\!\pm\!9.34\e{-6}$ & $\mathbf{7.57\e{-5}\!\pm\!9.87\e{-6}}$ \\
& SEIRH & $1.73\e{-5}\!\pm\!2.61\e{-6}$ & $\mathbf{1.52\e{-5}\!\pm\!1.24\e{-6}}$
        & $4.22\e{-5}\!\pm\!6.04\e{-6}$ & $\mathbf{3.85\e{-5}\!\pm\!6.80\e{-6}}$ \\

\midrule
\multirow{2}{*}{Gaussian}
& SEIR  & $2.58\e{-5}\!\pm\!5.01\e{-6}$ & $\mathbf{2.19\e{-5}\!\pm\!1.57\e{-6}}$
        & $7.91\e{-5}\!\pm\!1.09\e{-5}$ & $\mathbf{7.57\e{-5}\!\pm\!1.01\e{-5}}$ \\
& SEIRH & $1.71\e{-5}\!\pm\!1.51\e{-6}$ & $\mathbf{1.48\e{-5}\!\pm\!1.64\e{-6}}$
        & $4.08\e{-5}\!\pm\!4.99\e{-6}$ & $\mathbf{3.72\e{-5}\!\pm\!5.19\e{-6}}$ \\

\midrule
\multirow{2}{*}{Student-$t$}
& SEIR  & $2.37\e{-5}\!\pm\!3.07\e{-6}$ & $\mathbf{2.07\e{-5}\!\pm\!1.29\e{-6}}$
        & $7.69\e{-5}\!\pm\!8.19\e{-6}$ & $\mathbf{7.38\e{-5}\!\pm\!7.69\e{-6}}$ \\
& SEIRH & $1.65\e{-5}\!\pm\!1.14\e{-6}$ & $\mathbf{1.36\e{-5}\!\pm\!1.37\e{-6}}$
        & $4.12\e{-5}\!\pm\!4.28\e{-6}$ & $\mathbf{3.68\e{-5}\!\pm\!4.14\e{-6}}$ \\

\bottomrule
\vspace{-20pt}
\end{tabular}}
\end{table}

\textbf{Incorporating Dynamical System Structure Knowledge.}
We evaluated the ability of \texttt{MAAT} to incorporate \emph{structural priors} derived from known system semantics.  We consider the SEIR and SEIRH epidemiological models, where domain knowledge permits enforcing conservation of mass, non-negativity of all compartments, and monotonicity constraints implied by irreversible transitions (specifically, $R'(t)\ge 0$ and $S'(t)\le 0$). By injecting this physical prior into the system, we can substantially reduce the MSE of the recovered trajectories of  \texttt{MAAT}  (see Table \ref{tab:maat_priors_seir}).

\textbf{Real-world data modeling.}
We further evaluate our method on a real-world epidemiological dataset from the COVID-19 pandemic. As shown in Table~\ref{tab:covid-test-mse-ci}, \texttt{MAAT} achieves substantially lower reconstruction error compared to all baselines when coupled with SINDy, demonstrating its effectiveness in recovering realistic disease dynamics from noisy observational data.

\section{DISCUSSION}
We present \texttt{MAAT}, a framework for knowledge-informed state reconstruction from heterogeneous partial observations. By combining kernel-based continuous-time reconstruction with domain-specific observation operators and priors, \texttt{MAAT} provides physically meaningful trajectories and analytic derivatives for downstream mechanistic analysis, including symbolic regression.

\begin{table}[!htb]
\centering
\scriptsize
\caption{\textbf{COVID-19 benchmark (SINDy).} Test MSE across 5 seeds (mean $\pm$ 95\% CI).}
\label{tab:covid-test-mse-ci}
\setlength{\tabcolsep}{6pt}
\begin{tabular}{lcc}
\toprule
\textbf{Method} & \textbf{Test MSE} & \textbf{95\% CI} \\
\midrule
\texttt{MAAT} & $\mathbf{6.33\e{-5}}$ & $\pm\,1.07\e{-5}$ \\
RBF & $9.64\e{-4}$ & $\pm\,6.51\e{-4}$ \\
Savitzky--Golay & $9.73\e{-4}$ & $\pm\,6.47\e{-4}$ \\
TVRegDiff & $9.73\e{-4}$ & $\pm\,6.47\e{-4}$ \\
Linear & $9.80\e{-4}$ & $\pm\,6.53\e{-4}$ \\
Kalman filter & $9.89\e{-4}$ & $\pm\,6.68\e{-4}$ \\
Cubic & $9.99\e{-4}$ & $\pm\,6.72\e{-4}$ \\
Gaussian Process & $6.92\e{-2}$ & $\pm\,4.55\e{-2}$ \\
\bottomrule
\end{tabular}
\end{table}

\textbf{Limitations.}
The effectiveness of \texttt{MAAT} depends on the informativeness of the available observations and the validity of the injected priors. If measurements are too sparse, observation operators are uninformative, or priors are misspecified, reconstruction may be biased or underdetermined relative to purely data-driven smoothers. Our current implementation also assumes that relevant structural information, such as subsystem organization or admissible constraints, is known \emph{a priori}. While this is natural in domains such as quantitative systems pharmacology (QSP), where modular structure is often defined through absorption, distribution, receptor binding, and response subsystems \citep{Kaddi2018}, this assumption may not hold in less structured settings. 

\textbf{Nonlinear/Uncertain observation operators \& Future Work} \texttt{MAAT} works under the assumption that observation operators are known and linear. While this assumption is practically reasonable in many cases, the introduction of nonlinear operators or operators subject to uncertainty may enrich the spectrum of the possible phenomena that can be captured. We note, however, that nonlinearity can be introduced in \texttt{MAAT} sacrificing the convenience of the regression problem in Matrix Form. In fact if we define $\mathcal H : \mathbb R^d \to \mathbb R^m$ as a nonlinear operator acting on $\boldsymbol{x}$, the parametrization of $\widehat{\boldsymbol{x}}$ through control variables $\{u_{ij}\}_{ij}$ is still transparent to optimization, although regularity conditions on $\mathcal H$ and a priori knowledge on its nature are needed.

\textbf{Clinical and Translational Impact.}
A central motivation for \texttt{MAAT} is its alignment with how mechanistic modeling is practiced in medicine and pharmacology. In QSP, models are typically assembled from interpretable modules that describe biological and pharmacological processes, which are then coupled to explain patient-level outcomes \citep{Helmlinger2019}. \texttt{MAAT} supports this workflow by fusing sparse, subsystem-specific measurements with denser indirect signals while enforcing clinically meaningful constraints. This makes it well suited to model-informed drug development, where transparent and physiologically plausible models are needed for dose selection, safety evaluation, and regulatory decision-making \citep{FDA_MIDD_Program,EMA_PBPK_2018,Peterson2015}.

\section*{Impact Statement}
This paper presents work whose goal is to advance the field of Machine
Learning. There are many potential societal consequences of our work, none
which we feel must be specifically highlighted here.

\bibliography{main}

@article{doumeche2025physics,
  title={Physics-informed kernel learning},
  author={Doum{\`e}che, Nathan and Bach, Francis and Biau, G{\'e}rard and Boyer, Claire},
  journal={Journal of Machine Learning Research},
  volume={26},
  number={124},
  pages={1--39},
  year={2025}
}

@article{virgolin2022symbolic,
  title        = {Symbolic Regression is {NP\!-\!hard}},
  author       = {Virgolin, Marco and Pissis, Solon P.},
  journal      = {Transactions on Machine Learning Research},
  year         = {2022},
  note         = {Preprint on arXiv:2207.01018},
}

@misc{dugan2020occamnet,

  author = {Dugan,  Owen and Dangovski,  Rumen and Costa,  Allan and Kim,  Samuel and Goyal,  Pawan and Jacobson,  Joseph and Soljačić,  Marin},
  title = {OccamNet: A Fast Neural Model for Symbolic Regression at Scale},
  publisher = {arXiv},
  year = {2020},
  copyright = {arXiv.org perpetual,  non-exclusive license}
}

@article{makke2024interpretable,
  title = {Interpretable scientific discovery with symbolic regression: a review},
  volume = {57},
  ISSN = {1573-7462},
  number = {1},
  journal = {Artificial Intelligence Review},
  publisher = {Springer Science and Business Media LLC},
  author = {Makke,  Nour and Chawla,  Sanjay},
  year = {2024},
  month = jan 
}

@article{Bjrnstad2020,
  title = {The SEIRS model for infectious disease dynamics},
  volume = {17},
  ISSN = {1548-7105},
  number = {6},
  journal = {Nature Methods},
  publisher = {Springer Science and Business Media LLC},
  author = {Bjørnstad,  Ottar N. and Shea,  Katriona and Krzywinski,  Martin and Altman,  Naomi},
  year = {2020},
  month = jun,
  pages = {557–558}
}

@article{Brunton2016,
  title = {Discovering governing equations from data by sparse identification of nonlinear dynamical systems},
  volume = {113},
  ISSN = {1091-6490},
  number = {15},
  journal = {Proceedings of the National Academy of Sciences},
  publisher = {Proceedings of the National Academy of Sciences},
  author = {Brunton,  Steven L. and Proctor,  Joshua L. and Kutz,  J. Nathan},
  year = {2016},
  month = mar,
  pages = {3932–3937}
}

@article{Rackauckas2020,
  title = {Universal Differential Equations for Scientific Machine Learning},
  publisher = {Springer Science and Business Media LLC},
  author = {Rackauckas,  Christopher and Ma,  Yingbo and Martensen,  Julius and Warner,  Collin and Zubov,  Kirill and Supekar,  Rohit and Skinner,  Dominic and Ramadhan,  Ali and Edelman,  Alan},
  year = {2020},
  month = aug 
}

@article{Zhai2025,
  title = {Bridging known and unknown dynamics by transformer-based machine-learning inference from sparse observations},
  volume = {16},
  ISSN = {2041-1723},
  number = {1},
  journal = {Nature Communications},
  publisher = {Springer Science and Business Media LLC},
  author = {Zhai,  Zheng-Meng and Stern,  Benjamin D. and Lai,  Ying-Cheng},
  year = {2025},
  month = aug 
}

@inproceedings{rubanova2019latent,
  title={Latent ordinary differential equations for irregularly-sampled time series},
  author={Rubanova, Yulia and Chen, Ricky TQ and Duvenaud, David},
  booktitle={Advances in Neural Information Processing Systems},
  volume={32},
  year={2019}
}

@inproceedings{Chen2018neuralode,
 author = {Chen, Ricky T. Q. and Rubanova, Yulia and Bettencourt, Jesse and Duvenaud, David K},
 booktitle = {Advances in Neural Information Processing Systems},
 editor = {S. Bengio and H. Wallach and H. Larochelle and K. Grauman and N. Cesa-Bianchi and R. Garnett},
 pages = {},
 publisher = {Curran Associates, Inc.},
 title = {Neural Ordinary Differential Equations},
 volume = {31},
 year = {2018}
}

@article{Helmlinger2019,
  title = {Quantitative Systems Pharmacology: An Exemplar Model‐Building Workflow With Applications in Cardiovascular,  Metabolic,  and Oncology Drug Development},
  volume = {8},
  ISSN = {2163-8306},
  number = {6},
  journal = {CPT: Pharmacometrics \&amp; Systems Pharmacology},
  publisher = {Wiley},
  author = {Helmlinger,  Gabriel and Sokolov,  Victor and Peskov,  Kirill and Hallow,  Karen M. and Kosinsky,  Yuri and Voronova,  Veronika and Chu,  Lulu and Yakovleva,  Tatiana and Azarov,  Ivan and Kaschek,  Daniel and Dolgun,  Artem and Schmidt,  Henning and Boulton,  David W. and Penland,  Robert C.},
  year = {2019},
  month = jun,
  pages = {380–395}
}

@article{Peterson2015,
  title = {FDA Advisory Meeting Clinical Pharmacology Review Utilizes a Quantitative Systems Pharmacology (QSP) Model: A Watershed Moment?},
  volume = {4},
  ISSN = {2163-8306},
  number = {3},
  journal = {CPT: Pharmacometrics \&amp; Systems Pharmacology},
  publisher = {Wiley},
  author = {Peterson,  MC and Riggs,  MM},
  year = {2015},
  month = mar,
  pages = {189–192}
}

@inproceedings{
kacprzyk2025no,
title={No Equations Needed: Learning System Dynamics Without Relying on Closed-Form {ODE}s},
author={Krzysztof Kacprzyk and Mihaela van der Schaar},
booktitle={The Thirteenth International Conference on Learning Representations},
year={2025},
}

@article{Rudy2017,
  title = {Data-driven discovery of partial differential equations},
  volume = {3},
  ISSN = {2375-2548},
  number = {4},
  journal = {Science Advances},
  publisher = {American Association for the Advancement of Science (AAAS)},
  author = {Rudy,  Samuel H. and Brunton,  Steven L. and Proctor,  Joshua L. and Kutz,  J. Nathan},
  year = {2017},
  month = apr 
}

@article{Kaheman2020,
  title = {SINDy-PI: a robust algorithm for parallel implicit sparse identification of nonlinear dynamics},
  volume = {476},
  ISSN = {1471-2946},
  number = {2242},
  journal = {Proceedings of the Royal Society A: Mathematical,  Physical and Engineering Sciences},
  publisher = {The Royal Society},
  author = {Kaheman,  Kadierdan and Kutz,  J. Nathan and Brunton,  Steven L.},
  year = {2020},
  month = oct 
}

@article{Champion2019,
  title = {Data-driven discovery of coordinates and governing equations},
  volume = {116},
  ISSN = {1091-6490},
  number = {45},
  journal = {Proceedings of the National Academy of Sciences},
  publisher = {Proceedings of the National Academy of Sciences},
  author = {Champion,  Kathleen and Lusch,  Bethany and Kutz,  J. Nathan and Brunton,  Steven L.},
  year = {2019},
  month = oct,
  pages = {22445–22451}
}

@article{Lu2022,
  title = {Discovering sparse interpretable dynamics from partial observations},
  volume = {5},
  ISSN = {2399-3650},
  number = {1},
  journal = {Communications Physics},
  publisher = {Springer Science and Business Media LLC},
  author = {Lu,  Peter Y. and Ariño Bernad,  Joan and Soljačić,  Marin},
  year = {2022},
  month = aug 
}

@article{Raissi2019,
  title = {Physics-informed neural networks: A deep learning framework for solving forward and inverse problems involving nonlinear partial differential equations},
  volume = {378},
  ISSN = {0021-9991},
  journal = {Journal of Computational Physics},
  publisher = {Elsevier BV},
  author = {Raissi,  M. and Perdikaris,  P. and Karniadakis,  G.E.},
  year = {2019},
  month = feb,
  pages = {686–707}
}

@misc{Cranmer2023,
  author = {Cranmer,  Miles},
  title = {Interpretable Machine Learning for Science with PySR and SymbolicRegression.jl},
  publisher = {arXiv},
  year = {2023},
  copyright = {arXiv.org perpetual,  non-exclusive license}
}

@misc{Stephens2019gplearn,
  author       = {Trevor Stephens},
  title        = {gplearn: Genetic Programming in Python},
  year         = {2019},
  note         = {First released 2015},
  howpublished = {\url{https://gplearn.readthedocs.io/en/stable/index.html}}
}

@inproceedings{
qian2022dcode,
title={D-{CODE}: Discovering Closed-form {ODE}s from Observed Trajectories},
author={Zhaozhi Qian and Krzysztof Kacprzyk and Mihaela van der Schaar},
booktitle={International Conference on Learning Representations},
year={2022},
}

@inproceedings{Holt2022laplace,
  author={Samuel Holt and Zhaozhi Qian and Mihaela van der Schaar},
  title={Neural Laplace: Learning diverse classes of differential equations in the Laplace domain},
  year={2022},
  cdate={1640995200000},
  pages={8811-8832},
  booktitle={ICML},
 
}

@inproceedings{
holt2024datadriven,
title={Data-Driven Discovery of Dynamical Systems in Pharmacology using Large Language Models},
author={Samuel Holt and Zhaozhi Qian and Tennison Liu and Jim Weatherall and Mihaela van der Schaar},
booktitle={The Thirty-eighth Annual Conference on Neural Information Processing Systems},
year={2024},
}

@inproceedings{ dascoli2024odeformer, title={{ODEF}ormer: Symbolic Regression of Dynamical Systems with Transformers}, author={St{\'e}phane d'Ascoli and S{\"o}ren Becker and Philippe Schwaller and Alexander Mathis and Niki Kilbertus}, booktitle={The Twelfth International Conference on Learning Representations}, year={2024}, url={https://openreview.net/forum?id=TzoHLiGVMo} }

@inproceedings{ holt2023deep, title={Deep Generative Symbolic Regression}, author={Samuel Holt and Zhaozhi Qian and Mihaela van der Schaar}, booktitle={The Eleventh International Conference on Learning Representations }, year={2023},  }

@article{Kaddi2018,
  title = {Quantitative Systems Pharmacology Modeling of Acid Sphingomyelinase Deficiency and the Enzyme Replacement Therapy Olipudase Alfa Is an Innovative Tool for Linking Pathophysiology and Pharmacology},
  volume = {7},
  ISSN = {2163-8306},
  number = {7},
  journal = {CPT: Pharmacometrics \&amp; Systems Pharmacology},
  publisher = {Wiley},
  author = {Kaddi,  Chanchala D. and Niesner,  Bradley and Baek,  Rena and Jasper,  Paul and Pappas,  John and Tolsma,  John and Li,  Jing and van Rijn,  Zachary and Tao,  Mengdi and Ortemann‐Renon,  Catherine and Easton,  Rachael and Tan,  Sharon and Puga,  Ana Cristina and Schuchman,  Edward H. and Barrett,  Jeffrey S. and Azer,  Karim},
  year = {2018},
  month = jun,
  pages = {442–452}
}

@misc{FDA_MIDD_Program,
  author       = {{U.S. FDA}},
  title        = {Model-Informed Drug Development (MIDD) Paired Meeting Program},
  year         = {2023},
  howpublished = {\url{https://www.fda.gov/drugs/development-resources/model-informed-drug-development-paired-meeting-program}},
  note         = {Accessed 2025-10-02}
}

@misc{EMA_PBPK_2018,
  author       = {{European Medicines Agency}},
  title        = {Guideline on the Qualification and Reporting of Physiologically Based Pharmacokinetic (PBPK) Modelling and Simulation},
  year         = {2018},
  howpublished = {\url{https://www.ema.europa.eu/en/documents/scientific-guideline/draft-guideline-qualification-and-reporting-physiologically-based-pharmacokinetic-pbpk-modelling-and-simulation_en.pdf}},
  note         = {Accessed 2025-10-02}
}

@article{Steinier1972,
  title = {Smoothing and differentiation of data by simplified least square procedure},
  volume = {44},
  ISSN = {1520-6882},
  number = {11},
  journal = {Analytical Chemistry},
  publisher = {American Chemical Society (ACS)},
  author = {Steinier,  Jean. and Termonia,  Yves. and Deltour,  Jules.},
  year = {1972},
  month = sep,
  pages = {1906–1909}
}

@article{Chartrand2011,
  title = {Numerical Differentiation of Noisy,  Nonsmooth Data},
  volume = {2011},
  ISSN = {2090-5572},
  journal = {ISRN Applied Mathematics},
  publisher = {Wiley},
  author = {Chartrand,  Rick},
  year = {2011},
  month = may,
  pages = {1–11}
}

@book{deBoor1978,
author = {de Boor, Carl},
year = {1978},
month = {01},
pages = {},
title = {A Practical Guide to Splines},
volume = {Volume 27},
journal = {Applied Mathematical Sciences, New York: Springer, 1978},
doi = {10.2307/2006241}
}

@book{Buhmann2003,
  title = {Radial Basis Functions: Theory and Implementations},
  ISBN = {9780511543241},
  publisher = {Cambridge University Press},
  author = {Buhmann,  Martin D.},
  year = {2003},
  month = jul 
}

@article{Wang2021,
  title = {EXPLICIT ESTIMATION OF DERIVATIVES FROM DATA AND DIFFERENTIAL EQUATIONS BY GAUSSIAN PROCESS REGRESSION},
  volume = {11},
  ISSN = {2152-5080},
  number = {4},
  journal = {International Journal for Uncertainty Quantification},
  publisher = {Begell House},
  author = {Wang,  Hongqiao and Zhou,  X.},
  year = {2021},
  pages = {41–57}
}

@inproceedings{Hsin2025,
  title = {Symbolic Regression on Sparse and Noisy Data with Gaussian Processes},
  booktitle = {2025 American Control Conference (ACC)},
  publisher = {IEEE},
  author = {Hsin,  Junette and Agarwal,  Shubhankar and Thorpe,  Adam and Sentis,  Luis and Fridovich-Keil,  David},
  year = {2025},
  month = jul,
  pages = {3170–3175}
}

@InProceedings{Dondelinger2013,
  title = 	 {ODE parameter inference using adaptive gradient matching with Gaussian processes},
  author = 	 {Dondelinger, Frank and Husmeier, Dirk and Rogers, Simon and Filippone, Maurizio},
  booktitle = 	 {Proceedings of the Sixteenth International Conference on Artificial Intelligence and Statistics},
  pages = 	 {216--228},
  year = 	 {2013},
  editor = 	 {Carvalho, Carlos M. and Ravikumar, Pradeep},
  volume = 	 {31},
  series = 	 {Proceedings of Machine Learning Research},
  address = 	 {Scottsdale, Arizona, USA},
  month = 	 {29 Apr--01 May},
  publisher =    {PMLR}
}

@article{dePillis2014,
  title = {Mathematical Model of Colorectal Cancer with Monoclonal Antibody Treatments},
  volume = {4},
  ISSN = {2231-0614},
  number = {16},
  journal = {British Journal of Medicine and Medical Research},
  publisher = {Sciencedomain International},
  author = {dePillis,  L.},
  year = {2014},
  month = jan,
  pages = {3101–3131}
}

@article{Friberg2002,
  title = {Model of Chemotherapy-Induced Myelosuppression With Parameter Consistency Across Drugs},
  volume = {20},
  ISSN = {1527-7755},
  number = {24},
  journal = {Journal of Clinical Oncology},
  publisher = {American Society of Clinical Oncology (ASCO)},
  author = {Friberg,  Lena E. and Henningsson,  Anja and Maas,  Hugo and Nguyen,  Laurent and Karlsson,  Mats O.},
  year = {2002},
  month = dec,
  pages = {4713–4721}
}

@article{Kermack1927,
  volume = {115},
   title = {A contribution to the mathematical theory of epidemics},
  ISSN = {2053-9150},
  number = {772},
  journal = {Proceedings of the Royal Society of London. Series A,  Containing Papers of a Mathematical and Physical Character},
  publisher = {The Royal Society},
  year = {1927},
 author = {Kermack, William Ogilvy 
and McKendrick,  A. G.},
  month = aug,
  pages = {700–721}
}

@book{Nowak2000,
  title = {Virus dynamics: Mathematical principles of immunology and virology},
  ISBN = {9781383020816},
  publisher = {Oxford University PressOxford},
  author = {Nowak,  Martin A and May,  Robert M},
  year = {2000},
  month = nov 
}

@article{Sharma1998,
  title = {Characteristics of indirect pharmacodynamic models and applications to clinical drug responses},
  volume = {45},
  ISSN = {1365-2125},
  number = {3},
  journal = {British Journal of Clinical Pharmacology},
  publisher = {Wiley},
  author = {Sharma,  Amarnath and Jusko,  William J.},
  year = {1998},
  month = mar,
  pages = {229–239}
}

@article{Simeoni2004,
  title = {Predictive Pharmacokinetic-Pharmacodynamic Modeling of Tumor Growth Kinetics in Xenograft Models after Administration of Anticancer Agents},
  volume = {64},
  ISSN = {1538-7445},
  number = {3},
  journal = {Cancer Research},
  publisher = {American Association for Cancer Research (AACR)},
  author = {Simeoni,  Monica and Magni,  Paolo and Cammia,  Cristiano and De Nicolao,  Giuseppe and Croci,  Valter and Pesenti,  Enrico and Germani,  Massimiliano and Poggesi,  Italo and Rocchetti,  Maurizio},
  year = {2004},
  month = feb,
  pages = {1094–1101}
}

@article{Badiale2024,
  title = {A Nonlinear ODE Model for a Consumeristic Society},
  volume = {12},
  ISSN = {2227-7390},
  number = {8},
  journal = {Mathematics},
  publisher = {MDPI AG},
  author = {Badiale,  Marino and Cravero,  Isabella},
  year = {2024},
  month = apr,
  pages = {1253}
}

@article{Yu2024,
  title = {Learning dynamical systems from data: An introduction to physics-guided deep learning},
  volume = {121},
  ISSN = {1091-6490},
  number = {27},
  journal = {Proceedings of the National Academy of Sciences},
  publisher = {Proceedings of the National Academy of Sciences},
  author = {Yu,  Rose and Wang,  Rui},
  year = {2024},
  month = jun 
}

@book{strogatz2018nonlinear,
  title={Nonlinear Dynamics and Chaos},
  author={Strogatz, Steven H.},
  publisher={CRC Press},
  year={2018}
}

@article{Huang2013,
  title = {The Utility of Modeling and Simulation in Drug Development and Regulatory Review},
  volume = {102},
  ISSN = {0022-3549},
  number = {9},
  journal = {Journal of Pharmaceutical Sciences},
  publisher = {Elsevier BV},
  author = {Huang,  Shiew-Mei and Abernethy,  Darrell R. and Wang,  Yaning and Zhao,  Ping and Zineh,  Issam},
  year = {2013},
  month = sep,
  pages = {2912–2923}
}

@book{Rasmussen2005,
  title = {Gaussian Processes for Machine Learning},
  ISBN = {9780262256834},
  publisher = {The MIT Press},
  author = {Rasmussen,  Carl Edward and Williams,  Christopher K. I.},
  year = {2005},
  month = nov 
}

@article{Kalman1960,
  title = {A New Approach to Linear Filtering and Prediction Problems},
  volume = {82},
  ISSN = {0021-9223},
  number = {1},
  journal = {Journal of Basic Engineering},
  publisher = {ASME International},
  author = {Kalman,  R. E.},
  year = {1960},
  month = mar,
  pages = {35–45}
}

@article{Carnerero2023,
  title = {Kernel-Based State-Space Kriging for Predictive Control},
  volume = {10},
  ISSN = {2329-9274},
  number = {5},
  journal = {IEEE/CAA Journal of Automatica Sinica},
  publisher = {Institute of Electrical and Electronics Engineers (IEEE)},
  author = {Carnerero,  A. Daniel and Ramirez,  Daniel R. and Limon,  Daniel and Alamo,  Teodoro},
  year = {2023},
  month = may,
  pages = {1263–1275}
}

@inproceedings{Shakib2023,
  title = {Kernel-Based Learning of Stable Nonlinear State-Space Models},
  booktitle = {2023 62nd IEEE Conference on Decision and Control (CDC)},
  publisher = {IEEE},
  author = {Shakib,  M.F. and Tóth,  R. and Pogromsky,  A.Y. and Pavlov,  A. and van de Wouw,  N.},
  year = {2023},
  month = dec,
  pages = {2897–2902}
}

@article{Car2023,
  title = {Kernel Methods and Gaussian Processes for System Identification and Control: A Road Map on Regularized Kernel-Based Learning for Control},
  volume = {43},
  ISSN = {1941-000X},
  number = {5},
  journal = {IEEE Control Systems},
  publisher = {Institute of Electrical and Electronics Engineers (IEEE)},
  author = {Carè,  Algo and Carli,  Ruggero and Libera,  Alberto Dalla and Romeres,  Diego and Pillonetto,  Gianluigi},
  year = {2023},
  month = oct,
  pages = {69–110}
}

@article{Pillonetto2014,
  title = {Kernel methods in system identification,  machine learning and function estimation: A survey},
  volume = {50},
  ISSN = {0005-1098},
  number = {3},
  journal = {Automatica},
  publisher = {Elsevier BV},
  author = {Pillonetto,  Gianluigi and Dinuzzo,  Francesco and Chen,  Tianshi and De Nicolao,  Giuseppe and Ljung,  Lennart},
  year = {2014},
  month = mar,
  pages = {657–682}
}

@misc{Zhou2024,
  author = {Zhou,  Ying and Li,  Jinglai and Zhou,  Xiang and Wang,  Hongqiao},
  title = {Model-Embedded Gaussian Process Regression for Parameter Estimation in Dynamical System},
  publisher = {arXiv},
  year = {2024},
  copyright = {arXiv.org perpetual,  non-exclusive license}
}

@misc{Shojaee2024,

  author = {Shojaee,  Parshin and Meidani,  Kazem and Gupta,  Shashank and Farimani,  Amir Barati and Reddy,  Chandan K},
  title = {LLM-SR: Scientific Equation Discovery via Programming with Large Language Models},
  publisher = {arXiv},
  year = {2024},
  copyright = {Creative Commons Attribution 4.0 International}
}

@inproceedings{
Froese2023,
title={Training Neural Networks is {NP}-Hard in Fixed Dimension},
author={Vincent Froese and Christoph Hertrich},
booktitle={Thirty-seventh Conference on Neural Information Processing Systems},
year={2023},
}

@article{Brucker2023,
  title = {A Grey-box Model with Neural Ordinary Differential Equations for the Slow Voltage Dynamics of Lithium-ion Batteries: Model Development and Training},
  volume = {170},
  ISSN = {1945-7111},
  number = {12},
  journal = {Journal of The Electrochemical Society},
  publisher = {The Electrochemical Society},
  author = {Brucker,  Jennifer and Bessler,  Wolfgang G. and Gasper,  Rainer},
  year = {2023},
  month = dec,
  pages = {120537}
}

@inproceedings{
qian2021integrating,
title={Integrating Expert {ODE}s into Neural {ODE}s: Pharmacology and Disease Progression},
author={Zhaozhi Qian and William R. Zame and Lucas M Fleuren and Paul Elbers and Mihaela van der Schaar},
booktitle={Advances in Neural Information Processing Systems},
editor={A. Beygelzimer and Y. Dauphin and P. Liang and J. Wortman Vaughan},
year={2021},
}

@INPROCEEDINGS{Mehta2021,
  author={Mehta, Viraj and Char, Ian and Neiswanger, Willie and Chung, Youngseog and Nelson, Andrew and Boyer, Mark and Kolemen, Egemen and Schneider, Jeff},
  booktitle={2021 60th IEEE Conference on Decision and Control (CDC)}, 
  title={Neural Dynamical Systems: Balancing Structure and Flexibility in Physical Prediction}, 
  year={2021},
  volume={},
  number={},
  pages={3735-3742},
  doi={10.1109/CDC45484.2021.9682807}}

@article{Sivapalan2023,
  title = {Liquid biopsy approaches to capture tumor evolution and clinical outcomes during cancer immunotherapy},
  volume = {11},
  ISSN = {2051-1426},
  number = {1},
  journal = {Journal for ImmunoTherapy of Cancer},
  publisher = {BMJ},
  author = {Sivapalan,  Lavanya and Murray,  Joseph C and Canzoniero,  Jenna VanLiere and Landon,  Blair and Jackson,  Jennifer and Scott,  Susan and Lam,  Vincent and Levy,  Benjamin P. and Sausen,  Mark and Anagnostou,  Valsamo},
  year = {2023},
  month = Jan,
  pages = {e005924}
}

@article{Bartolomucci2025,
  title = {Circulating tumor DNA to monitor treatment response in solid tumors and advance precision oncology},
  volume = {9},
  ISSN = {2397-768X},
  number = {1},
  journal = {npj Precision Oncology},
  publisher = {Springer Science and Business Media LLC},
  author = {Bartolomucci,  Alexandra and Nobrega,  Monyse and Ferrier,  Tadhg and Dickinson,  Kyle and Kaorey,  Nivedita and Nadeau,  Amélie and Castillo,  Alberto and Burnier,  Julia V.},
  year = {2025},
  month = Mar 
}

@misc{ecdc_covid19_data,
  author       = {{European Centre for Disease Prevention and Control}},
  title        = {Data on daily new cases of COVID-19 in EU/EEA countries},
  year         = {2020},
  howpublished = {\url{https://www.ecdc.europa.eu/en/publications-data/data-daily-new-cases-covid-19-eueea-country}},
  note         = {Accessed: 2026-04-29}
}
\bibliographystyle{icml2026}

\newpage
\appendix
\onecolumn
\icmltitle{Supplementary Materials for \texttt{MAAT}}

\section{Extended Related Works}\label{app:extended_works}
\paragraph{Sample and Computational Complexity of Noisy Symbolic Regression}  
Symbolic regression aims to identify an expression \(f\) from a library of primitives (variables, constants, operators, functions) that explains a dataset \(\{(x_i, y_i)\}_{i=1}^N\). The challenge lies in the size of the hypothesis space: the number of candidate expressions grows combinatorially with both the input dimension \(d\) and the allowed expression depth. This renders even state-of-the-art search procedures computationally demanding. Recent theoretical work provides formal support for this intuition: \cite{virgolin2022symbolic} show that symbolic regression is \(NP\)-hard under mild assumptions on the primitive set, implying that no polynomial-time algorithm exists unless \(P = NP\). These hardness results highlight why practical SR algorithms must rely on heuristic search, structural priors, or regularization, particularly in the presence of noise.

Concretely, one can reduce a canonical \(NP\)-hard problem (e.g., \textsc{Subset Sum} or additive partitioning) to the decision problem of whether there exists a symbolic expression of bounded cost that fits the data within a specified error tolerance. Consequently, any algorithm that guarantees recovery of the globally optimal symbolic representation must, in the worst case, incur non-polynomial runtime in \(d\). In practice, even heuristic or stochastic search procedures face the combinatorial explosion of the grammar: each additional feature or operator multiplies the number of candidate sub-expressions. Recent approaches (e.g., neural-symbolic methods \citep{dugan2020occamnet}, transformer-based architectures, or reinforcement-learning–guided search \citep{makke2024interpretable}) mitigate this via pruning, learned proposals, or modular decomposition. However, these strategies do not fundamentally escape the exponential scaling in \(d\) (or in the maximum depth) unless further structural assumptions, such as sparsity, separability, or modular factorization, are imposed.

\paragraph{Classical equation discovery.}
Recovering governing equations from data has a long history in system identification and sparse regression. 
Sparse Identification of Nonlinear Dynamics (SINDy) \citep{Brunton2016} recovers parsimonious differential equations from feature libraries, with extensions for noise \citep{Rudy2017} and implicit formulations \citep{Kaheman2020}. 
Genetic programming frameworks such as gplearn \citep{Stephens2019gplearn} and evolutionary engines like PySR \citep{Cranmer2023} broaden the search space beyond fixed polynomial libraries, recovering compact analytical expressions directly from data. 
More recently, complementary approaches replace genetic heuristics with modern deep learning: ODEFormer \citep{dascoli2024odeformer} leverages transformer architectures for symbolic regression of dynamical systems, while Deep Generative Symbolic Regression \citep{holt2023deep} employs generative models to efficiently explore the space of candidate equations. 
Together, these methods expand the toolkit for symbolic discovery, but they, too, typically assume full observability.

\paragraph{Black-box dynamical models.}
Neural differential equation frameworks such as Neural ODEs \citep{Chen2018neuralode}, DyNODE, and Latent ODEs \citep{rubanova2019latent} model vector fields with neural networks, enabling scalable learning under partial observability. 
 However, this surrogate–distillation pipeline can degrade or collapse under substantial observation noise or partial observability, where errors in the learned latent representation propagate into the recovered symbolic form. These methods achieve strong predictive performance but lack closed-form interpretability. Nevertheless, they can serve as surrogates from which symbolic models are distilled, suggesting complementarity rather than exclusivity between black-box and symbolic approaches. A growing body of work therefore studies \emph{grey-box} or \emph{hybrid} neural ODEs,
which incorporate prior knowledge such as known state variables, conservation laws,
stability constraints, or partially specified equations into neural dynamics
\citep{Mehta2021, Brucker2023, qian2021integrating}. These approaches improve
identifiability and inductive bias relative to unconstrained Neural ODEs, but typically
retain neural components in the governing equations, limiting closed-form interpretability
and robustness under severe noise or sparse observations.

\paragraph{Hybrid and recent approaches.}
Universal Differential Equations (UDEs) \citep{Rackauckas2020} combine mechanistic ODEs with neural components, embedding domain priors into flexible models but without mechanisms for handling fragmented heterogeneous data or structured interventions. 
More recent directions move beyond time-domain formulations: Neural Laplace \citep{Holt2022laplace} learns in the Laplace domain to capture diverse dynamics at scale, and D-CODE \citep{qian2022dcode} targets exact closed-form recovery. 
Large language models further extend this space: recent work \citep{holt2024datadriven, Shojaee2024} shows that LLMs can propose function libraries, encode domain knowledge, and refine candidate models, pointing toward a synthesis of symbolic, neural, and language-based methods.

\paragraph{Equation discovery under partial observations.}
In many scientific domains, full-state measurements are rarely available: biological, clinical, and physical systems are often only partially observed, corrupted by noise, or measured at mismatched time resolutions. 
This has motivated methods that seek to augment incomplete data with synthetic simulations or structural priors. 
For example, \citet{Zhai2025} show that supplementing scarce real measurements with synthetic trajectories can improve the identifiability of governing equations in clinical settings. 
Other approaches leverage probabilistic latent-variable models to infer hidden dynamics before applying regression or discovery methods \citep{Champion2019,Lu2022}. 
A complementary strategy is to encode known physical structure into the learning process: physics-informed neural networks (PINNs) enforce differential equation constraints within the training loss, enabling recovery of latent variables and dynamics even when only a subset of states are measured \citep{Raissi2019}. 
Despite these advances, most existing work assumes homogeneous data sources and does not explicitly address the integration of heterogeneous, subsystem-specific observations, an issue particularly acute in fields like pharmacology or systems biology, where data are fragmented across scales.

\paragraph{Equation-free yet interpretable modeling.}
Beyond explicit equation discovery, several recent approaches aim to balance predictive power with interpretability without producing closed-form equations. 
Operator-learning methods represent dynamical evolution directly through learned integral operators, enabling transparent structural analysis even in the absence of symbolic expressions  \citep{Kacprzyk2025no}. 
Similarly, mechanistic machine learning frameworks embed known conservation laws or monotonicity constraints into neural architectures, yielding models whose behavior can be interrogated and trusted despite lacking closed-form governing equations. 
These approaches illustrate that interpretability need not always hinge on explicit symbolic laws, but can also arise through structural constraints and transparent operator representations.

\paragraph{GP smoothing for trajectories .}
Classical Gaussian Processes (GPs) provide nonparametric, smooth interpolants with analytic derivative posteriors via kernel differentiation \citep{Rasmussen2005}. When all (or most) states are observed on a single grid, a GP prior over trajectories can effectively denoise and supply \(\dot{x}(t)\) estimates. However, this paradigm does not natively address \emph{partial observability}, \emph{multi-rate heterogeneous sensors}, or \emph{hierarchical priors} on cross-subsystem couplings. Works that explicitly target derivative recovery with GPs leverage the fact that derivatives of a GP are again GPs, enabling closed-form posterior means/variances for \(\dot{x}\) \citep{Wang2021}. These methods excel as denoisers/preprocessors but typically assume either fully (or densely) observed states and do not resolve the inverse problem of reconstructing latent states from images of linear observation operators \(H_i x\) collected at different rates.

A direct line of work fits GPs to each observed channel (or a multi-output GP), extracts analytic derivatives, and then applies symbolic regression (SR) such as SINDy/PySR \citep{Hsin2025}. This reduces noise sensitivity in SR but still treats observation fusion as a \emph{per-channel smoothing} step. Handling heterogeneous observation models, enforcing block-structured priors, or constraining library couplings generally requires bespoke kernel designs and does not provide an explicit mechanism for hierarchy-aware discovery. Adaptive Gradient Matching (AGM) uses GPs to interpolate states and matches ODE right-hand sides to GP derivatives to infer parameters without direct numerical integration \citep{Dondelinger2013}. AGM targets \emph{parameter estimation for a specified ODE family}, not discovery of new symbolic structure. 

\paragraph{Kernel-based state reconstruction.}
A related line of work uses kernel methods to reconstruct latent dynamical states from indirect, noisy, or irregular observations. Early system identification approaches framed state reconstruction as a regularized inverse problem, using reproducing kernel Hilbert space (RKHS) priors to interpolate trajectories under smoothness or stability assumptions \citep{Pillonetto2014}. Subsequent work extended these ideas to multi-output and operator-valued kernels, enabling joint reconstruction of multiple state components and improved handling of correlated signals \citep{Carnerero2023,Shakib2023}. More recent methods consider kernel-based formulations for learning continuous-time dynamics from sparse or asynchronous measurements, often emphasizing well-posedness and statistical consistency of the reconstruction step \citep{Zhou2024,Car2023}. 

While these approaches provide powerful tools for denoising and interpolating partially observed trajectories, they are typically developed for \emph{state estimation} or \emph{parameter identification} rather than symbolic structure discovery. In particular, most kernel-based reconstruction methods assume either a single observation operator or homogeneous sensing modalities, and they do not explicitly address the integration of heterogeneous, subsystem-specific observations or the imposition of hierarchical sparsity over downstream model components. Moreover, kernel reconstructions are usually treated as preprocessing steps, without a principled mechanism for propagating reconstruction uncertainty or structural priors into subsequent equation discovery. \texttt{MAAT} builds on this kernel perspective but departs from prior work by explicitly coupling kernel state reconstruction with symbolic regression, enabling joint handling of heterogeneous partial observability and structured discovery of governing dynamics.

\texttt{MAAT} is designed for \emph{heterogeneous partial observability} and \emph{structural discovery}. Its kernel state reconstruction (KSR) jointly fits multiple observation operators \(H_i\) and sampling grids to recover a coherent latent trajectory \(x(t)\) and analytic \(\dot{x}(t)\) \emph{before} SR, and then injects physics-aware priors via library masks to constrain cross-subsystem couplings. An overview of the different preprocessing pipelines and their properties is given in Table \ref{tab:estimators-capabilities}.

\begin{table*}[!t]
\centering
\scriptsize
\renewcommand{\arraystretch}{1.2} 
\caption{\textbf{Capability matrix for state and derivative estimators.} Comparison of classical numerical methods, probabilistic estimators, and deep learning approaches. \texttt{MAAT} satisfies the requirements for equation discovery under partial observability and structured priors.}
\label{tab:estimators-capabilities}
\resizebox{\textwidth}{!}{%
\begin{tabular}{lcccccc}
\toprule
\textbf{Method} &
  \textbf{Analytic $\dot{x}$} &
  \textbf{Handles Irregular Data} &
  \textbf{Noise Robust} &
  \textbf{Fuses Heterog. Obs.} &
  \textbf{Supports Structure} &
  \textbf{Supports Priors} \\ 
\midrule
\multicolumn{7}{l}{\textit{Numerical \& Smoothing Baselines}} \\ 
Finite Differences & \cbad  & \cbad  & \cbad  & \cbad & \cbad  & \cbad \\
Savitzky--Golay    & \cpart & \cbad  & \cpart & \cbad & \cbad  & \cbad \\
TVRegDiff          & \cbad  & \cpart & \cgood & \cbad & \cbad  & \cpart \\
Cubic Spline       & \cgood & \cpart & \cpart & \cbad & \cbad  & \cpart \\ 
RBF Kernels        & \cgood & \cgood & \cpart & \cbad & \cbad  & \cbad \\
\midrule
\multicolumn{7}{l}{\textit{Probabilistic \& Filtering}} \\ 
Gaussian Processes & \cgood & \cgood & \cgood & \cpart & \cpart & \cgood \\
Kalman Filters     & \cpart & \cgood & \cgood & \cgood & \cbad  & \cgood \\ 
\midrule
\multicolumn{7}{l}{\textit{Deep Learning \& Physics-Informed}} \\ 
Neural ODEs           & \cgood & \cgood & \cgood & \cpart & \cbad & \cbad \\
\midrule
\rowcolor{blue!5} 
\textbf{\texttt{MAAT} (Ours)} & \cgood & \cgood & \cgood & \cgood & \cgood & \cgood \\ 
\bottomrule
\end{tabular}%
}
\vspace{-15pt}
\end{table*}

\section{Theoretical analysis}\label{app:theory}
In this appendix we present the theoretical foundations underlying our state reconstruction algorithm. We outline the motivation for its design, provide the analytical derivation, and discuss the results that justify its adoption within our framework.
\subsection{Proof of Lemma 1}

\begin{Definition}[Equivalence Between Distances]
\label{def:equiv}
Let $(X,d)$ and $(X,d')$ be two metric spaces defined on the same underlying set $X$ through two different metrics $d: X \times X \rightarrow \mathbb R_+$ and $d': X \times X \rightarrow \mathbb R_+$. We say that $d$ and $d'$ are \emph{equivalent}, written $d \sim d'$, if there exist constants $c,C > 0$ such that
\begin{equation}
    \exists \ c,C \in \mathbb R_+ : c d(x_1,x_2) \le d'(x_1,x_2) \le C d(x_1,x_2) \ \ \ \forall x_1,x_2 \in X
\end{equation}
\end{Definition}

\begin{Corollary} Let $d,d'$ be two metrics such that $d \sim d'$. Then for any point $x^* \in X$ and any sequence $\{x_n\}_n^{\infty} \subset X$ such that
\begin{equation}
    \lim_{n \rightarrow \infty} d(x_n,x^*) = 0
\end{equation}
it also holds that
\begin{equation}
    \lim_{n \rightarrow \infty} d'(x_n,x^*) = 0
\end{equation}
and vice versa.
\end{Corollary}
\label{cor:equiv}
\begin{proof}
The forward direction follows immediately by the squeeze theorem. 
For the reverse implication, note that by assumption
\begin{equation}
    d(x_1,x_2) \;\le\; \tfrac{1}{c}\, d'(x_1, x_2) \;\le\; \tfrac{C}{c}\, d(x_1,x_2),
\end{equation}
and similarly
\begin{equation}
    \tfrac{c}{C}\, d(x_1,x_2) \;\le\; \tfrac{1}{C}\, d'(x_1, x_2) \;\le\; d(x_1,x_2).
\end{equation}
Hence
\begin{equation}
    \tfrac{1}{C}\, d'(x_1, x_2)
        \;\le\; d(x_1,x_2)
        \;\le\; \tfrac{1}{c}\, d'(x_1, x_2),
\end{equation}
showing the equivalence is symmetric, and the claim follows.
\end{proof}

\begin{Lemma}[Composite loss is a calibrated surrogate]
\label{lemma:comp}
Let $H:\mathbb{R}^d\!\to\!\mathbb{R}^p$ be a bounded linear observation operator with operator norm $\|H\| < \infty$.
For any candidate trajectory $\hat{x}\in L^2([0,T];\mathbb{R}^d)$ and true trajectory $x$, define the risk functional
\[
\mathcal{R}(\hat{x})=\|x-\hat{x}\|_{L^2}^2+\|H(x-\hat{x})\|_{L^2}^2.
\]
Then
\[
\|x-\hat{x}\|_{L^2}^2 \;\le\; \mathcal{R}(\hat{x}) \;\le\; (1+\|H\|^2)\,\|x-\hat{x}\|_{L^2}^2.
\]
\end{Lemma}
\begin{proof}
The lower bound is immediate, as $\|H(x-\hat{x})\|_{L^2}^2 \ge 0$.
For the upper bound, we use the definition of the induced operator norm for $H$:
\[
\|H\boldsymbol{v}\|_2 \le \|H\| \|\boldsymbol{v}\|_2 \quad \forall \boldsymbol{v} \in \mathbb{R}^d.
\]
Applying this to the $L^2$ norm of the function $H(x(t)-\hat{x}(t))$:
\begin{align*}
\|H(x-\hat{x})\|_{L^2}^2 &= \int_0^T \|H(x(t)-\hat{x}(t))\|_2^2 dt \\
&\le \int_0^T \|H\|^2 \|x(t)-\hat{x}(t)\|_2^2 dt \\
&= \|H\|^2 \int_0^T \|x(t)-\hat{x}(t)\|_2^2 dt \\
&= \|H\|^2 \|x-\hat{x}\|_{L^2}^2.
\end{align*}
Substituting this result into the definition of $\mathcal{R}(\hat{x})$ gives:
\begin{align*}
\mathcal{R}(\hat{x}) &= \|x-\hat{x}\|_{L^2}^2+\|H(x-\hat{x})\|_{L^2}^2 \\
&\le \|x-\hat{x}\|_{L^2}^2 + \|H\|^2 \|x-\hat{x}\|_{L^2}^2 \\
&= (1 + \|H\|^2) \|x-\hat{x}\|_{L^2}^2.
\end{align*}
Thus $\mathcal{R}(\hat{x})$ is bounded above and below by constant multiples of the true error $\|x-\hat{x}\|_{L^2}^2$. 
By Definition~\ref{def:equiv} and Corollary~\ref{cor:equiv}, minimizing $\mathcal{R}(\hat{x})$ is equivalent for interpolating solutions to minimizing the true $L^2$ error, establishing the surrogate calibration.
\end{proof}

\subsection{Proof Sketch for Proposition 1}
\begin{Proposition}[FD noise floor vs KSR]
Assume additive i.i.d.\ zero-mean noise $\epsilon(t)$ with variance $\sigma^2$ on measurements of $x(t)$ sampled with step size $\Delta t$. 
Using central finite differences (FD) to approximate $\dot{x}(t)$ yields a mean-squared derivative error of:
$\mathbb{E}\!\left[\|\widehat{\dot{x}}_{\mathrm{FD}}-\dot{x}\|_2^2\right]
= \mathcal{O}(\Delta t^4)+\Omega(\sigma^2/\Delta t^2)$.
For KSR (kernel ridge with a twice-differentiable kernel) with regularization $\lambda$
and $n$ samples, the analytic derivative estimator satisfies
$\mathbb{E}\!\left[\|\widehat{\dot{x}}_{\mathrm{KSR}}-\dot{x}\|_2^2\right]
= \mathcal{O}(\lambda) + \mathcal{O}(\sigma^2/n)$.
\end{Proposition}
\begin{proof}[Sketch]
\textbf{Finite Differences (FD)}: The central difference approximation for $\dot{x}(t_i)$ is $\widehat{\dot{x}}_{\mathrm{FD}}(t_i) = \frac{x(t_{i+1}) + \epsilon_{i+1} - (x(t_{i-1}) + \epsilon_{i-1})}{2\Delta t}$. The error has two components: a bias term from the Taylor expansion, which is $\mathcal{O}(\Delta t^2)$, and a variance term from the noise. The squared error, in the case $\mathbb E[\epsilon_i] = 0$, is:
\[
\mathbb{E}\left[ \left( \frac{x(t_{i+1}) - x(t_{i-1})}{2\Delta t} - \dot{x}(t_i) \right)^2 + \left( \frac{\epsilon_{i+1} - \epsilon_{i-1}}{2\Delta t} \right)^2 \right].
\]
The squared bias is $\mathcal{O}(\Delta t^4)$. The variance term is $\mathbb{E}\left[ \frac{\epsilon_{i+1}^2 - 2\epsilon_{i+1}\epsilon_{i-1} + \epsilon_{i-1}^2}{4\Delta t^2} \right] = \frac{2\sigma^2}{4\Delta t^2} = \frac{\sigma^2}{2\Delta t^2}$.
Thus, the total MSE is $\mathcal{O}(\Delta t^4) + \Omega(\sigma^2/\Delta t^2)$. As $\Delta t \to 0$, the variance term explodes.

\textbf{Kernel State Reconstruction (KSR)}:
Taking expectation over the noise realizations and applying the standard bias--variance decomposition, we obtain

 \begin{align} \mathbb E[ || \hat {\dot x} - \dot x||_2^2 ] & = \sum_i^n \mathsf{Var} [\hat{\dot x}_i] + || \mathbb E [\hat{\dot x}] - \dot x||^2_2 \end{align}

We bound the two terms separately.

Because the Gaussian kernel is smooth, differentiation is a bounded operator and we obtain, for some constant $C$,
$$
|| \mathbb E [\hat{\dot x}] - \dot x||^2_2
\le C
|| \mathbb E [\hat{\dot x}] - \dot x||^2_\mathcal H = \mathcal O(\lambda)
$$

By Lemma 2,
$$
 \hat x - \mathbb E \hat x = k(t_i)^\top (\mathbf K + \lambda \mathbf I)^{-1} \boldsymbol \epsilon
$$
for a matrix $\mathbf K$ induced by the observation operators, with $\boldsymbol\epsilon$ zero-mean noise.

The variance can then be expressed as

\begin{align}
\mathsf{Var} [\hat {\dot x}_i ] 
& = \mathbb E \left\{ 
\boldsymbol \epsilon^\top
 (\mathbf K + \lambda \mathbf I)^{-1} 
 \dot k(t_i) \dot k(t_i)^\top  (\mathbf K + \lambda \mathbf I)^{-1} \boldsymbol \epsilon
\right\}
\\
& =  \sigma_{noise}^2 \dot k(t_i)^\top  (\mathbf K + \lambda \mathbf I)^{-2} 
 \dot k(t_i).
\end{align}

Differentiating the Gaussian kernel with respect to time gives

\begin{align}
\dot k(t)^\top \dot k(t)
&  =
\sum_{t_i} \left( \frac{t - t_i}{\sigma^2} \right)^2 e^{ - \frac{||t - t_i ||^2}{\sigma^2}}
= \mathcal O(n).
\end{align}

We also note that
$$
	(\mathbf K + \lambda \mathbf I)^{-2} = \frac{1}{n^2} \left(
	\frac{1}{n}\mathbf K + \frac{1}{n}\lambda \mathbf I
	 \right)^{-2}
$$
and since, for any matrix $\mathbf A$ we have
$$
\boldsymbol v^\top \mathbf A \boldsymbol v \sim \mathsf{Tr}[\mathbf A] || \boldsymbol v ||^2
$$
the $n^{-1}$ factor acts as a normalization constant.

Hence,
$$
\mathsf{Var}[\hat{\dot x}_i] \sim {\sigma_{noise}^2 \underbrace{n}_{\mathcal O ||\dot k||}} \frac{1}{n^2}  \mathsf{Tr}\left(
	\frac{1}{n}\mathbf K + \frac{1}{n}\lambda \mathbf I
	 \right)^{-2} = \mathcal O\left(\frac{\sigma^2_{noise}}{n}\right).
$$
proving that
$$
 \mathbb E[ || \hat {\dot x} - \dot x||_2^2 ] 
 =  \mathcal O(\lambda) + \mathcal O\left(\frac{\sigma^2_{noise}}{n}\right) 
$$

\end{proof}

\subsection{Method derivation}

Let $({t}^{obs}, X^{obs}) \in \mathbb R^{N_{obs}} \times \mathbb R^{N_{obs} \times D}$ denote the set of full (possibly noisy) observations of the state trajectory $x(t) \in \mathbb R^D$. 
Each row of $X^{obs}$ corresponds to one observed state, recorded at the corresponding entry in ${t}^{obs}$.  

In contrast, let $Y \in \mathbb R^{N \times S}$ represent a collection of $S$ observed signals, each measured on a common vector of time points $t = (t_1,\ldots,t_N) \in \mathbb R^N$. 
Each column of $Y$ corresponds to one signal. 
We assume these signals arise through a linear observation operator $H \in \mathbb R^{S \times D}$ applied to the (unknown) full state matrix $X \in \mathbb R^{N \times D}$:
\begin{equation}
    Y = X H^\top.
\end{equation}

Our goal in this setting is to construct a differentiable, vector-valued function 
\[
\widehat{{x}} : \mathbb R \;\to\; \mathbb R^D
\]
that approximates the true state trajectory $\boldsymbol{x}(t)$. 
We measure accuracy via the expected squared error
\begin{equation}
   \mathbb E_{t \sim \mathcal T}
    \big\|
    \widehat{{x}}(t) - {{x}}(t)
    \big\|^2
    =
    \int_{\mathrm{supp}(\mathcal T)}
    p_{\mathcal T}(t)\,
    \big\|
    \widehat{{x}}(t) - {{x}}(t)
    \big\|^2 \, dt,
\end{equation}
where $\mathcal T$ denotes the distribution of observation times. 
Assuming $\mathcal T$ is uniform, this reduces to a rescaled $L^2$ distance between the reconstructed and true trajectories.  

From the perspective of statistical learning theory, this task is equivalent to learning a function that maps points from a real interval (time) into a higher-dimensional space $\mathbb R^D$, while accounting for partial and noisy observations.

In light of such consideration and our assumption of a sparse-observations regime ($N_{obs} \ll N$), a further regularization step is required to estimate the state in a meaningful way. This result can be obtained both  by incorporating information on the fidelity of the signal (i.e. how the signal, which is fully observed, is reconstructed by the application of the operator $ H$ on the predicted state) and by imposing further regularization constraints that may come from domain knowledge related to the specific nature of the dynamical system. Formally, we obtain the regularized risk functional

\begin{equation}
\mathcal R_{reg}(\hat {x})
:=
\mathbb E_{t \sim \mathcal T}
\left\|
\widehat{x}(t) - x(t)
\right\|^2
+
\mathbb E_{t \sim \mathcal T}
\left\|
H \widehat{x}(t) - Hx(t)
\right\|^2    
+ 
\lambda \,\mathfrak R(\widehat{{x}} \mid \mathscr C),
\end{equation}

where $\mathfrak R$ denotes a regularization operator and $\mathscr C$ represents the contextual knowledge used in constructing the regularization. In the absence of regularization ($\lambda = 0$), this reduces to the case described by Lemma~\ref{lemma:comp}, from which we infer the induced equivalence between the two risk measures.

To formulate the regression problem associated with the constructed risk measure, we express $\widehat{{x}}$ as a linear combination of features obtained via a map $\phi: \mathbb R \rightarrow \mathcal H$ into a Hilbert space $\mathcal H$. Under this hypothesis, we can reparameterize the risk measure as
\begin{equation}
\begin{split}
\mathcal R_{reg}(\{w_j\}_j) =
\mathbb E_{t \sim \mathcal T}
\sum_{j=1}^D
\left(
\phi(t)^\top {w}_j  - x_j(t)
\right)^2 
+
\mathbb E_{t \sim \mathcal T}
\sum_{s=1}^S
\left(
\sum_{j=1}^D
H_{sj}\,\phi(t)^\top {w}_j  - y_s(t)
\right)^2
+ 
\lambda \,\mathfrak R(\{w_j\}_j \mid \mathscr C).
\end{split}
\end{equation}

Replacing expectations by empirical averages yields
\begin{equation}
\begin{split}
\widehat{\mathcal R}_{reg}(\{w_j\}_j)
&= 
\frac{1}{N_{obs}} \sum_{i=1}^{N_{obs}}
\sum_{j=1}^D
\Bigl(
\phi(t^{obs}_i)^\top {w}_j - X_{ij}^{obs}
\Bigr)^2 \\[0.5em]
&\quad +
\frac{1}{N} \sum_{i=1}^N
\sum_{s=1}^S
\Bigl(
\sum_{j=1}^D H_{sj}\,\phi(t_i)^\top {w}_j - Y_{is}
\Bigr)^2
+ 
\lambda \,\mathfrak R(
\{w_j\}_j
\mid \mathscr C).
\end{split}
\end{equation}

Since $\phi$ may be infinite-dimensional, we restrict ${w}_1,\ldots,{w}_D$ to the span of $\{\phi(t_i)\}_{i=1}^N$, writing
\begin{equation}
{w}_j = \sum_{\ell=1}^N u_{\ell j}\,\phi(t_\ell), \qquad j=1,\ldots,D,
\end{equation}
for coefficients $\{u_{\ell j}\} \subset \mathbb R$ arranged in a matrix $U \in \mathbb R^{N \times D}$. Using the kernel trick $\phi(x)^\top \phi(y) = \kappa(x,y)$, the empirical loss becomes
\begin{equation}
\begin{split}
\widehat{\mathcal R}_{reg}( U)
&= 
\frac{1}{N_{obs}} \sum_{i=1}^{N_{obs}}
\sum_{j=1}^D
\left(
\sum_{\ell=1}^N u_{\ell j}\,\kappa(t^{obs}_i,t_\ell) - X_{ij}^{obs}
\right)^2 \\[0.5em]
&\quad +
\frac{1}{N} \sum_{i=1}^N
\sum_{s=1}^S
\left(
\sum_{j=1}^D \sum_{\ell=1}^N H_{sj}\,u_{\ell j}\,\kappa(t_i,t_\ell) - Y_{is}
\right)^2
+ 
\lambda \,\mathfrak R(U \mid \mathscr C).
\end{split}
\end{equation}

Defining kernel matrices $[ K^{obs}]_{i\ell} := \kappa(t^{obs}_i,t_\ell)$ and $[K]_{i\ell} := \kappa(t_i,t_\ell)$, we obtain the compact form
\begin{equation}
\widehat{\mathcal R}_{reg}(U)
=
\frac{1}{N_{obs}}
\left\|
 K^{obs} U -  X^{obs}
\right\|_F^2
+
\frac{1}{N}
\left\|
KUH^\top -  Y
\right\|_F^2
+
\lambda \,\mathfrak R(U \mid \mathscr C).
\end{equation}

with the reconstructed state being, again in matrix form
\begin{equation}
    \widehat{ X} = K U.
\end{equation} 
Interestingly, the choice of the Gaussian feature map allows explicit computation of the derivatives which can be employed in the construction of context aware constraints in $\mathcal R$. In fact for any differential operator $T$ we obtain
\begin{align*}
    T \widehat{x_j}(t)
    & = 
    T \left\{ \phi(t)^\top \sum_i^N \phi(t_i) u_{ij} \right\}
    \\
    & = 
    T \left\{ \sum_i^N \kappa(t,t_i) u_{ij} \right\}
    \\
    & = 
    \sum_i^N  T \kappa(t,t_i) u_{ij} 
\end{align*}
with $T \kappa$, the image of the kernel $\kappa$ through the linear operator $T$, being analytically computable for the case of the Gaussian feature map.

We note that our method intrinsically supports the inclusion of additional loss terms, which can encode prior knowledge or structural constraints. Such terms may be specified directly by domain experts or automatically suggested by large language models. For example, one may penalize large deviations in the dynamics by introducing a regularizer of the form
\begin{equation}
    \mathfrak R(U) \;=\;
    \|\dot{K} U\|_F^2,
\end{equation}
where $\dot{ K}$ denotes the Gram matrix after applying the time-derivative operator. 
More generally, this mechanism allows the integration of expert priors (e.g., known critical points, concavity properties, or monotonicity constraints) into the reconstruction process. 
In this way, our framework mimics the role of a human statistician in guiding model specification, while retaining the flexibility to incorporate data-driven or automatically generated hypotheses.

\section{Extended Experimental Results}\label{app:extended_results}
\subsection{Additional noise regimes}

Tables~\ref{tab:maat_selected_dataset_results_correlated} and~\ref{tab:maat_selected_dataset_results_student} further demonstrate that \texttt{MAAT} consistently outperforms competing baselines under both correlated Gaussian noise and heavy-tailed Student-$t$ noise.

\begin{table*}[h]
\centering
\caption{State reconstruction MSE ($\downarrow$) semi-synthetic benchmark datasets. Values are mean $\pm$ confidence interval. Best result for each dataset-backend pair is bolded. Noise type: Correlated Gaussian.}
\label{tab:maat_selected_dataset_results_correlated}
\tiny
\setlength{\tabcolsep}{2.0pt}
\renewcommand{\arraystretch}{0.92}
\resizebox{\textwidth}{!}{%
\begin{tabular}{llccccccccc}
\toprule
\multirow{2}{*}{Method} & \multirow{2}{*}{Backend} & \multicolumn{3}{c}{Dynamical systems} & \multicolumn{3}{c}{Epidemiology / dynamics} & \multicolumn{3}{c}{Oncology / viral} \\
\cmidrule(lr){3-5}\cmidrule(lr){6-8}\cmidrule(lr){9-11}
& & CRC & Cons. & Neut. & SEIR & SEIRH & TMDD & Tumor & TDI & Viral \\
\midrule
\multirow{2}{*}{RBF}
& PySR & $4.3\e{-2}\!\pm\!5.6\e{-3}$ & $3.3\e{1}\!\pm\!5.7\e{0}$ & $8.8\e{-3}\!\pm\!7.8\e{-3}$ & $2.9\e{-3}\!\pm\!2.8\e{-4}$ & $2.7\e{-3}\!\pm\!2.6\e{-4}$ & $3.4\e{-1}\!\pm\!4.7\e{-2}$ & $3.1\e{0}\!\pm\!3.9\e{-1}$ & $1.8\e{1}\!\pm\!4.1\e{0}$ & $2.6\e{-3}\!\pm\!2.4\e{-4}$ \\
& SINDy & $4.3\e{-2}\!\pm\!5.7\e{-3}$ & $3.6\e{1}\!\pm\!5.7\e{0}$ & $8.7\e{-3}\!\pm\!7.6\e{-3}$ & $3.0\e{-3}\!\pm\!2.7\e{-4}$ & $2.7\e{-3}\!\pm\!2.6\e{-4}$ & $3.4\e{-1}\!\pm\!4.6\e{-2}$ & $3.1\e{0}\!\pm\!3.6\e{-1}$ & $1.9\e{1}\!\pm\!4.1\e{0}$ & $2.7\e{-3}\!\pm\!2.2\e{-4}$ \\
\addlinespace[1.2pt]
\multirow{2}{*}{Cubic}
& PySR & $5.6\e{-2}\!\pm\!8.4\e{-3}$ & $4.2\e{1}\!\pm\!6.3\e{0}$ & $1.2\e{-2}\!\pm\!1.1\e{-2}$ & $3.8\e{-3}\!\pm\!4.4\e{-4}$ & $3.6\e{-3}\!\pm\!4.2\e{-4}$ & $4.5\e{-1}\!\pm\!6.5\e{-2}$ & $4.0\e{0}\!\pm\!5.1\e{-1}$ & $3.6\e{1}\!\pm\!7.3\e{0}$ & $3.5\e{-3}\!\pm\!3.3\e{-4}$ \\
& SINDy & $5.6\e{-2}\!\pm\!8.5\e{-3}$ & $4.5\e{1}\!\pm\!6.6\e{0}$ & $1.2\e{-2}\!\pm\!1.1\e{-2}$ & $3.9\e{-3}\!\pm\!4.4\e{-4}$ & $3.7\e{-3}\!\pm\!4.2\e{-4}$ & $4.5\e{-1}\!\pm\!6.5\e{-2}$ & $4.0\e{0}\!\pm\!5.2\e{-1}$ & $3.7\e{1}\!\pm\!7.5\e{0}$ & $3.6\e{-3}\!\pm\!3.2\e{-4}$ \\
\addlinespace[1.2pt]
\multirow{2}{*}{GP}
& PySR & $2.3\e{-1}\!\pm\!2.5\e{-1}$ & $2.3\e{2}\!\pm\!2.1\e{2}$ & $4.8\e{-2}\!\pm\!4.7\e{-2}$ & $5.0\e{-2}\!\pm\!8.7\e{-2}$ & $7.7\e{-3}\!\pm\!7.7\e{-3}$ & $9.2\e{-2}\!\pm\!4.1\e{-2}$ & $4.0\e{1}\!\pm\!4.1\e{1}$ & $1.9\e{3}\!\pm\!3.2\e{3}$ & $2.6\e{-2}\!\pm\!5.8\e{-2}$ \\
& SINDy & $4.1\e{-1}\!\pm\!2.5\e{-1}$ & $2.4\e{2}\!\pm\!3.3\e{2}$ & $2.0\e{-1}\!\pm\!4.7\e{-1}$ & $5.9\e{-3}\!\pm\!6.0\e{-3}$ & $2.2\e{-3}\!\pm\!2.2\e{-3}$ & $8.0\e{-2}\!\pm\!3.4\e{-2}$ & $2.9\e{1}\!\pm\!2.5\e{1}$ & $2.1\e{1}\!\pm\!1.9\e{1}$ & $1.9\e{-2}\!\pm\!1.4\e{-2}$ \\
\addlinespace[1.2pt]
\multirow{2}{*}{Kalman}
& PySR & $1.2\e{-2}\!\pm\!1.6\e{-3}$ & $\mathbf{1.0\e{1}}\!\mathbf{\pm}\!\mathbf{2.9\e{0}}$ & $2.6\e{-3}\!\pm\!2.4\e{-3}$ & $7.7\e{-4}\!\pm\!1.3\e{-4}$ & $7.3\e{-4}\!\pm\!1.3\e{-4}$ & $8.8\e{-2}\!\pm\!1.1\e{-2}$ & $9.1\e{-1}\!\pm\!1.4\e{-1}$ & $1.8\e{1}\!\pm\!3.8\e{0}$ & $7.0\e{-4}\!\pm\!7.4\e{-5}$ \\
& SINDy & $1.2\e{-2}\!\pm\!1.7\e{-3}$ & $1.3\e{1}\!\pm\!2.8\e{0}$ & $2.6\e{-3}\!\pm\!2.3\e{-3}$ & $8.3\e{-4}\!\pm\!1.3\e{-4}$ & $7.5\e{-4}\!\pm\!1.3\e{-4}$ & $8.8\e{-2}\!\pm\!1.1\e{-2}$ & $9.2\e{-1}\!\pm\!1.4\e{-1}$ & $1.8\e{1}\!\pm\!4.0\e{0}$ & $7.9\e{-4}\!\pm\!6.7\e{-5}$ \\
\addlinespace[1.2pt]
\multirow{2}{*}{Linear}
& PySR & $2.8\e{-2}\!\pm\!3.8\e{-3}$ & $2.4\e{1}\!\pm\!4.2\e{0}$ & $6.0\e{-3}\!\pm\!5.5\e{-3}$ & $1.9\e{-3}\!\pm\!2.1\e{-4}$ & $1.8\e{-3}\!\pm\!2.0\e{-4}$ & $2.2\e{-1}\!\pm\!3.1\e{-2}$ & $2.1\e{0}\!\pm\!2.6\e{-1}$ & $1.7\e{1}\!\pm\!3.7\e{0}$ & $1.7\e{-3}\!\pm\!1.6\e{-4}$ \\
& SINDy & $2.8\e{-2}\!\pm\!3.9\e{-3}$ & $2.5\e{1}\!\pm\!4.3\e{0}$ & $6.0\e{-3}\!\pm\!5.3\e{-3}$ & $2.0\e{-3}\!\pm\!2.1\e{-4}$ & $1.8\e{-3}\!\pm\!2.0\e{-4}$ & $2.2\e{-1}\!\pm\!3.1\e{-2}$ & $2.1\e{0}\!\pm\!2.7\e{-1}$ & $1.8\e{1}\!\pm\!3.9\e{0}$ & $1.8\e{-3}\!\pm\!1.5\e{-4}$ \\
\addlinespace[1.2pt]
\multirow{2}{*}{NeuralODE}
& PySR & $6.2\e{1}\!\pm\!1.3\e{2}$ & $4.8\e{12}\!\pm\!1.1\e{13}$ & $9.4\e{0}\!\pm\!2.0\e{1}$ & $1.1\e{0}\!\pm\!6.1\e{-1}$ & $4.8\e{-1}\!\pm\!1.4\e{-1}$ & $1.9\e{0}\!\pm\!2.6\e{0}$ & $8.6\e{4}\!\pm\!1.4\e{5}$ & $8.6\e{8}\!\pm\!1.9\e{9}$ & $6.0\e{-1}\!\pm\!3.0\e{-1}$ \\
& SINDy & $1.8\e{0}\!\pm\!8.5\e{-1}$ & $1.1\e{3}\!\pm\!1.0\e{3}$ & $7.4\e{-1}\!\pm\!7.5\e{-1}$ & $8.0\e{-1}\!\pm\!6.7\e{-1}$ & $3.2\e{-1}\!\pm\!1.2\e{-1}$ & $2.2\e{0}\!\pm\!3.1\e{0}$ & $2.4\e{2}\!\pm\!2.5\e{2}$ & $4.1\e{3}\!\pm\!9.2\e{3}$ & $5.2\e{-1}\!\pm\!3.1\e{-1}$ \\
\midrule
\multirow{2}{*}{\textbf{\texttt{MAAT}}}
& \multicolumn{1}{>{\columncolor{gray!7}}c}{PySR}
& \multicolumn{1}{>{\columncolor{gray!7}}c}{$\mathbf{3.6\e{-3}\!\pm\!1.9\e{-3}}$}
& \multicolumn{1}{>{\columncolor{gray!7}}c}{$1.1\e{1}\!\pm\!1.3\e{1}$}
& \multicolumn{1}{>{\columncolor{gray!7}}c}{$\mathbf{3.4\e{-4}\!\pm\!3.5\e{-4}}$}
& \multicolumn{1}{>{\columncolor{gray!7}}c}{$\mathbf{2.4\e{-5}\!\pm\!4.0\e{-6}}$}
& \multicolumn{1}{>{\columncolor{gray!7}}c}{$\mathbf{1.7\e{-5}\!\pm\!2.6\e{-6}}$}
& \multicolumn{1}{>{\columncolor{gray!7}}c}{$\mathbf{3.4\e{-2}\!\pm\!3.7\e{-2}}$}
& \multicolumn{1}{>{\columncolor{gray!7}}c}{$\mathbf{3.9\e{-1}\!\pm\!2.8\e{-1}}$}
& \multicolumn{1}{>{\columncolor{gray!7}}c}{$\mathbf{5.4\e{0}\!\pm\!5.9\e{0}}$}
& \multicolumn{1}{>{\columncolor{gray!7}}c}{$\mathbf{4.1\e{-5}\!\pm\!1.9\e{-5}}$} \\
& \multicolumn{1}{>{\columncolor{gray!7}}c}{SINDy}
& \multicolumn{1}{>{\columncolor{gray!7}}c}{$\mathbf{1.4\e{-3}\!\pm\!1.3\e{-4}}$}
& \multicolumn{1}{>{\columncolor{gray!7}}c}{$\mathbf{5.8\e{0}\!\pm\!4.9\e{0}}$}
& \multicolumn{1}{>{\columncolor{gray!7}}c}{$\mathbf{4.1\e{-4}\!\pm\!4.6\e{-4}}$}
& \multicolumn{1}{>{\columncolor{gray!7}}c}{$\mathbf{7.9\e{-5}\!\pm\!9.3\e{-6}}$}
& \multicolumn{1}{>{\columncolor{gray!7}}c}{$\mathbf{4.2\e{-5}\!\pm\!6.0\e{-6}}$}
& \multicolumn{1}{>{\columncolor{gray!7}}c}{$\mathbf{4.8\e{-3}\!\pm\!6.3\e{-4}}$}
& \multicolumn{1}{>{\columncolor{gray!7}}c}{$\mathbf{1.7\e{-1}\!\pm\!6.1\e{-2}}$}
& \multicolumn{1}{>{\columncolor{gray!7}}c}{$\mathbf{1.7\e{0}\!\pm\!4.8\e{-1}}$}
& \multicolumn{1}{>{\columncolor{gray!7}}c}{$\mathbf{1.3\e{-4}\!\pm\!2.5\e{-5}}$} \\
\bottomrule
\end{tabular}}
\end{table*}

\begin{table*}[h]
\centering
\caption{State reconstruction MSE ($\downarrow$) semi-synthetic benchmark datasets. Values are mean $\pm$ confidence interval. Best result for each dataset-backend pair is bolded. Noise type: Student T.}
\label{tab:maat_selected_dataset_results_student}
\tiny
\setlength{\tabcolsep}{2.0pt}
\renewcommand{\arraystretch}{0.92}
\resizebox{\textwidth}{!}{%
\begin{tabular}{llccccccccc}
\toprule
\multirow{2}{*}{Method} & \multirow{2}{*}{Backend} & \multicolumn{3}{c}{Dynamical systems} & \multicolumn{3}{c}{Epidemiology / dynamics} & \multicolumn{3}{c}{Oncology / viral} \\
\cmidrule(lr){3-5}\cmidrule(lr){6-8}\cmidrule(lr){9-11}
& & CRC & Cons. & Neut. & SEIR & SEIRH & TMDD & Tumor & TDI & Viral \\
\midrule
\multirow{2}{*}{RBF}
& PySR & $1.5\e{-1}\!\pm\!3.1\e{-2}$ & $1.2\e{2}\!\pm\!1.8\e{1}$ & $3.1\e{-2}\!\pm\!2.6\e{-2}$ & $1.1\e{-2}\!\pm\!2.3\e{-3}$ & $9.0\e{-3}\!\pm\!1.1\e{-3}$ & $1.3\e{0}\!\pm\!1.7\e{-1}$ & $1.1\e{1}\!\pm\!1.3\e{0}$ & $5.7\e{1}\!\pm\!1.3\e{1}$ & $8.4\e{-3}\!\pm\!1.4\e{-3}$ \\
& SINDy & $1.5\e{-1}\!\pm\!3.1\e{-2}$ & $1.2\e{2}\!\pm\!1.8\e{1}$ & $3.1\e{-2}\!\pm\!2.6\e{-2}$ & $1.1\e{-2}\!\pm\!2.3\e{-3}$ & $9.1\e{-3}\!\pm\!1.1\e{-3}$ & $1.3\e{0}\!\pm\!1.7\e{-1}$ & $1.1\e{1}\!\pm\!1.3\e{0}$ & $5.8\e{1}\!\pm\!1.3\e{1}$ & $8.5\e{-3}\!\pm\!1.4\e{-3}$ \\
\addlinespace[1.2pt]
\multirow{2}{*}{Cubic}
& PySR & $2.3\e{-1}\!\pm\!4.2\e{-2}$ & $1.9\e{2}\!\pm\!3.6\e{1}$ & $4.9\e{-2}\!\pm\!3.7\e{-2}$ & $1.6\e{-2}\!\pm\!3.4\e{-3}$ & $1.4\e{-2}\!\pm\!1.9\e{-3}$ & $2.1\e{0}\!\pm\!5.0\e{-1}$ & $1.9\e{1}\!\pm\!2.0\e{0}$ & $1.6\e{2}\!\pm\!5.0\e{1}$ & $1.4\e{-2}\!\pm\!3.5\e{-3}$ \\
& SINDy & $2.3\e{-1}\!\pm\!4.2\e{-2}$ & $1.9\e{2}\!\pm\!3.6\e{1}$ & $4.8\e{-2}\!\pm\!3.7\e{-2}$ & $1.6\e{-2}\!\pm\!3.4\e{-3}$ & $1.4\e{-2}\!\pm\!1.9\e{-3}$ & $2.1\e{0}\!\pm\!5.0\e{-1}$ & $1.9\e{1}\!\pm\!2.0\e{0}$ & $1.6\e{2}\!\pm\!5.0\e{1}$ & $1.5\e{-2}\!\pm\!3.5\e{-3}$ \\
\addlinespace[1.2pt]
\multirow{2}{*}{GP}
& PySR & $6.8\e{-1}\!\pm\!6.3\e{-1}$ & $1.9\e{2}\!\pm\!1.7\e{2}$ & $6.4\e{-2}\!\pm\!4.5\e{-2}$ & $1.3\e{-2}\!\pm\!1.3\e{-2}$ & $7.4\e{-3}\!\pm\!6.7\e{-3}$ & $\mathbf{4.3\e{-2}}\!\mathbf{\pm}\!\mathbf{1.2\e{-2}}$ & $2.7\e{1}\!\pm\!3.4\e{1}$ & $4.6\e{2}\!\pm\!4.8\e{2}$ & $3.6\e{-2}\!\pm\!4.9\e{-2}$ \\
& SINDy & $3.9\e{-1}\!\pm\!3.2\e{-1}$ & $3.2\e{2}\!\pm\!3.5\e{2}$ & $9.0\e{-2}\!\pm\!6.2\e{-2}$ & $5.2\e{-3}\!\pm\!4.6\e{-3}$ & $3.9\e{-3}\!\pm\!3.6\e{-3}$ & $6.4\e{-2}\!\pm\!4.1\e{-2}$ & $3.2\e{1}\!\pm\!2.5\e{1}$ & $7.4\e{2}\!\pm\!1.2\e{3}$ & $1.4\e{-2}\!\pm\!1.2\e{-2}$ \\
\addlinespace[1.2pt]
\multirow{2}{*}{Kalman}
& PySR & $1.1\e{-2}\!\pm\!2.5\e{-3}$ & $9.8\e{0}\!\pm\!2.3\e{0}$ & $2.6\e{-3}\!\pm\!2.5\e{-3}$ & $8.4\e{-4}\!\pm\!2.6\e{-4}$ & $6.6\e{-4}\!\pm\!1.1\e{-4}$ & $1.1\e{-1}\!\pm\!1.9\e{-2}$ & $8.0\e{-1}\!\pm\!1.6\e{-1}$ & $4.9\e{1}\!\pm\!1.1\e{1}$ & $6.5\e{-4}\!\pm\!1.3\e{-4}$ \\
& SINDy & $1.1\e{-2}\!\pm\!2.6\e{-3}$ & $1.3\e{1}\!\pm\!2.1\e{0}$ & $2.5\e{-3}\!\pm\!2.3\e{-3}$ & $8.9\e{-4}\!\pm\!2.6\e{-4}$ & $6.7\e{-4}\!\pm\!1.1\e{-4}$ & $1.0\e{-1}\!\pm\!1.9\e{-2}$ & $8.1\e{-1}\!\pm\!1.6\e{-1}$ & $4.9\e{1}\!\pm\!1.2\e{1}$ & $7.5\e{-4}\!\pm\!1.2\e{-4}$ \\
\addlinespace[1.2pt]
\multirow{2}{*}{Linear}
& PySR & $7.5\e{-2}\!\pm\!1.6\e{-2}$ & $6.4\e{1}\!\pm\!1.1\e{1}$ & $1.6\e{-2}\!\pm\!1.3\e{-2}$ & $5.7\e{-3}\!\pm\!1.3\e{-3}$ & $4.6\e{-3}\!\pm\!5.3\e{-4}$ & $6.7\e{-1}\!\pm\!9.8\e{-2}$ & $5.9\e{0}\!\pm\!6.6\e{-1}$ & $5.3\e{1}\!\pm\!1.2\e{1}$ & $4.3\e{-3}\!\pm\!7.9\e{-4}$ \\
& SINDy & $7.4\e{-2}\!\pm\!1.6\e{-2}$ & $6.5\e{1}\!\pm\!1.1\e{1}$ & $1.6\e{-2}\!\pm\!1.3\e{-2}$ & $5.7\e{-3}\!\pm\!1.3\e{-3}$ & $4.6\e{-3}\!\pm\!5.3\e{-4}$ & $6.7\e{-1}\!\pm\!9.7\e{-2}$ & $5.8\e{0}\!\pm\!6.5\e{-1}$ & $5.3\e{1}\!\pm\!1.2\e{1}$ & $4.4\e{-3}\!\pm\!7.8\e{-4}$ \\
\addlinespace[1.2pt]
\multirow{2}{*}{NeuralODE}
& PySR & $1.1\e{1}\!\pm\!1.7\e{1}$ & $2.1\e{10}\!\pm\!3.4\e{10}$ & $2.5\e{1}\!\pm\!6.5\e{1}$ & $9.6\e{-1}\!\pm\!5.9\e{-1}$ & $5.3\e{-1}\!\pm\!3.2\e{-1}$ & $1.0\e{0}\!\pm\!7.5\e{-1}$ & $7.5\e{7}\!\pm\!1.2\e{8}$ & $3.4\e{10}\!\pm\!7.5\e{10}$ & $1.7\e{0}\!\pm\!2.4\e{0}$ \\
& SINDy & $2.2\e{0}\!\pm\!2.9\e{0}$ & $4.1\e{3}\!\pm\!5.1\e{3}$ & $1.1\e{0}\!\pm\!9.1\e{-1}$ & $5.1\e{-1}\!\pm\!2.6\e{-1}$ & $4.5\e{-1}\!\pm\!2.2\e{-1}$ & $1.3\e{0}\!\pm\!1.1\e{0}$ & $2.6\e{2}\!\pm\!2.8\e{2}$ & $7.4\e{1}\!\pm\!8.8\e{1}$ & $7.9\e{-1}\!\pm\!6.4\e{-1}$ \\
\midrule
\multirow{2}{*}{\textbf{\texttt{MAAT}}}
& \multicolumn{1}{>{\columncolor{gray!7}}c}{PySR}
& \multicolumn{1}{>{\columncolor{gray!7}}c}{$\mathbf{5.6\e{-3}\!\pm\!4.1\e{-3}}$}
& \multicolumn{1}{>{\columncolor{gray!7}}c}{$\mathbf{5.1\e{0}\!\pm\!4.8\e{0}}$}
& \multicolumn{1}{>{\columncolor{gray!7}}c}{$\mathbf{3.6\e{-4}\!\pm\!4.8\e{-4}}$}
& \multicolumn{1}{>{\columncolor{gray!7}}c}{$\mathbf{2.4\e{-5}\!\pm\!3.1\e{-6}}$}
& \multicolumn{1}{>{\columncolor{gray!7}}c}{$\mathbf{1.7\e{-5}\!\pm\!1.1\e{-6}}$}
& \multicolumn{1}{>{\columncolor{gray!7}}c}{$4.3\e{-2}\!\pm\!2.6\e{-2}$}
& \multicolumn{1}{>{\columncolor{gray!7}}c}{$\mathbf{2.4\e{-1}\!\pm\!1.4\e{-1}}$}
& \multicolumn{1}{>{\columncolor{gray!7}}c}{$\mathbf{9.1\e{0}\!\pm\!1.4\e{1}}$}
& \multicolumn{1}{>{\columncolor{gray!7}}c}{$\mathbf{4.3\e{-5}\!\pm\!2.2\e{-5}}$} \\
& \multicolumn{1}{>{\columncolor{gray!7}}c}{SINDy}
& \multicolumn{1}{>{\columncolor{gray!7}}c}{$\mathbf{1.4\e{-3}\!\pm\!1.4\e{-4}}$}
& \multicolumn{1}{>{\columncolor{gray!7}}c}{$\mathbf{4.5\e{0}\!\pm\!2.0\e{0}}$}
& \multicolumn{1}{>{\columncolor{gray!7}}c}{$\mathbf{4.2\e{-4}\!\pm\!5.0\e{-4}}$}
& \multicolumn{1}{>{\columncolor{gray!7}}c}{$\mathbf{7.7\e{-5}\!\pm\!8.2\e{-6}}$}
& \multicolumn{1}{>{\columncolor{gray!7}}c}{$\mathbf{4.1\e{-5}\!\pm\!4.3\e{-6}}$}
& \multicolumn{1}{>{\columncolor{gray!7}}c}{$\mathbf{4.8\e{-3}\!\pm\!6.1\e{-4}}$}
& \multicolumn{1}{>{\columncolor{gray!7}}c}{$\mathbf{1.3\e{-1}\!\pm\!3.0\e{-2}}$}
& \multicolumn{1}{>{\columncolor{gray!7}}c}{$\mathbf{1.7\e{0}\!\pm\!4.2\e{-1}}$}
& \multicolumn{1}{>{\columncolor{gray!7}}c}{$\mathbf{1.2\e{-4}\!\pm\!2.3\e{-5}}$} \\
\bottomrule
\end{tabular}}
\end{table*}

\subsection{Computational Complexity}
\label{app:complexity}

A potential concern with kernel-based methods is the $\mathcal{O}(N^3)$ cost of direct Gram-matrix inversion. However, \texttt{MAAT} avoids this bottleneck entirely: our implementation relies on first-order optimization (Adam) rather than closed-form matrix inversion. Each iteration requires only matrix-vector products with the kernel matrix, reducing per-step complexity to $\mathcal{O}(N^2)$. This makes the method scalable to long or high-resolution trajectories. The situation is analogous to neural network training: although exact loss minimization is NP-hard in the worst case \citep{Froese2023}, practical gradient-based optimizers achieve good solutions efficiently.

Table~\ref{tab:wallclock} reports wall-clock time and peak memory usage across all baselines, averaged over the benchmark suite on identical hardware. \texttt{MAAT} incurs moderate overhead relative to classical smoothers (e.g., splines, Savitzky--Golay), but remains substantially faster than Neural ODEs. While slower than lightweight SINDy-based pipelines, it achieves consistently superior reconstruction accuracy. Its memory footprint is higher than simple interpolation methods but remains comparable to other kernel and neural approaches.

\begin{table}[h]
\centering
\small
\caption{\textbf{Computational cost comparison.} Mean wall-clock time and peak memory usage across all benchmarks, separated by downstream backend.}
\label{tab:wallclock}
\setlength{\tabcolsep}{6pt}
\begin{tabular}{lcc|cc}
\toprule
& \multicolumn{2}{c}{\textbf{PySR}} & \multicolumn{2}{c}{\textbf{SINDy}} \\
\cmidrule(lr){2-3}\cmidrule(lr){4-5}
\textbf{Method} & \textbf{Time (s)} & \textbf{Mem (MB)} & \textbf{Time (s)} & \textbf{Mem (MB)} \\
\midrule
Linear & 569.44 & 1331.03 & 3.99 & 220.72 \\
RBF & 594.47 & 1330.07 & 10.65 & 222.54 \\
Cubic & 575.01 & 1334.95 & 3.96 & 221.07 \\
Savitzky--Golay & 583.39 & 1336.63 & 4.35 & 222.14 \\
Kalman filter & 586.86 & 1348.64 & 9.56 & 230.78 \\
TVRegDiff & 591.17 & 1340.55 & 4.15 & 222.15 \\
Gaussian Process & 668.85 & 1332.21 & 149.01 & 232.16 \\
Neural ODE & 1165.29 & 1482.51 & 489.41 & 457.45 \\
\midrule
\texttt{MAAT} (ours) & 693.39 & 1445.34 & 183.87 & 322.86 \\
\bottomrule
\end{tabular}
\end{table}

\section{Experimental Details}\label{app:experimental_details}
Our evaluation suite consists of several ODE models commonly used in computational biology and pharmacology. We evaluate \texttt{MAAT} on both standard dynamical systems benchmarks and clinically motivated pharmacological models, reflecting the clinical relevance emphasized in our motivation. At the same time, we demonstrate the applicability of our method beyond these settings by including representative models from additional scientific domains.

\subsection{Dynamical systems}

\paragraph{Colorectal cancer model}
We adopt the seven-state CRC–mAb–IL2–chemo system modelling tumour burden ($T$), natural killer cells ($N$), lymphocytes ($L$), circulating chemotherapy ($C$), monoclonal antibody ($M$), cytokines ($I$), and a secondary antibody pool ($A$) \citep{dePillis2014}.  

\paragraph{Neutrophil life-cycle model}
The Friberg-type transit system tracks proliferating precursors (Prol), transit compartments (T$_1$, T$_2$, T$_3$), marrow reservoir (Reserv), and circulating neutrophils (Circ) \citep{Friberg2002}. 

\paragraph{Consumeristic socio-ecological model}
Population ($x$), renewable resources ($y$), non-renewable resources ($z$), and wealth ($w$) \citep{Badiale2024}.

\subsection{Epidemiology}

\paragraph{SEIR compartmental epidemic model}
Susceptible ($S$), exposed ($E$), infectious ($I$), and removed ($R$) populations evolve under homogeneous mixing \citep{Kermack1927}:

\paragraph{SEIRH epidemic model with hospitalisation}
Extending SEIR with a hospitalised class $H$ and transition rate $\delta$ \citep{Bjrnstad2020}.

\paragraph{Tumour PK/PD (TMDD + RO–driven TGI)}
We couple a 3-compartment monoclonal-antibody PK (central $C_c$, peripheral $C_p$, tumour $C_t$) with target-mediated drug disposition (TMDD) at tumour/peripheral sites, and drive tumour growth inhibition (TGI) by receptor occupancy (RO). 
 This construction is consistent with TMDD foundations and minimal/PBPK mAb models that include a tumour (interstitial) site and use RO as the pharmacodynamic driver of TGI \citep{Simeoni2004}. All tumour‐rate constants are converted from per-day to per-hour in implementation.

\subsection{Oncology/Viral dynamical systems}

\paragraph{Tumour model} Custom logistic model of tumor growth, studying the temporal evolution of its volume.

\begin{align}
\begin{cases}
\frac{dC_c}{dt} &= -\frac{CL}{V_c} C_c - \frac{Q_p}{V_c}(C_c - C_p) - \frac{Q_t}{V_c}(C_c - C_t) \\[6pt]
\frac{dC_p}{dt} &= \frac{Q_p}{V_p}(C_c - C_p) \\
\frac{dC_t}{dt} &= \frac{Q_t}{V_t}(C_c - C_t) \\
\frac{dR_A}{dt} &= -k_{\mathrm{on},A}\, C_t\, R_A + k_{\mathrm{off},A}\, (R_{A,\mathrm{tot}} - R_A) \\
\frac{dR_B}{dt} &= -k_{\mathrm{on},B}\, C_p\, R_B + k_{\mathrm{off},B}\, (R_{B,\mathrm{tot}} - R_B) \\
\frac{dT}{dt} &= K_g\, T \left(1 - \frac{T}{T_{\max}}\right)
- K_k \, \frac{R_{A,\mathrm{bound}}}{IC_{50} + R_{A,\mathrm{bound}}} \, T
\end{cases}
\end{align}
with
\begin{align}
C_c &:\ \text{Drug concentration in the central (plasma) compartment} \\[4pt]
C_p &:\ \text{Drug concentration in the peripheral compartment} \\[4pt]
C_t &:\ \text{Drug concentration in the tumour compartment} \\[6pt]
R_A &:\ \text{Unbound receptor A concentration} \\
R_B &:\ \text{Unbound receptor B concentration} \\
R_{A,\mathrm{tot}} &:\ \text{Total receptor A concentration} \\
R_{B,\mathrm{tot}} &:\ \text{Total receptor B concentration} \\
R_{A,\mathrm{bound}} &:\ \text{Bound receptor A concentration} = R_{A,\mathrm{tot}} - R_A \\[6pt]
T &:\ \text{Tumour volume} \\[10pt]
V_c &:\ \text{Central compartment volume} \\
V_p &:\ \text{Peripheral compartment volume} \\
V_t &:\ \text{Tumour compartment volume} \\[6pt]
CL &:\ \text{Clearance from central compartment} \\
Q_p &:\ \text{Inter-compartmental flow between central and peripheral} \\
Q_t &:\ \text{Inter-compartmental flow between central and tumour} \\[6pt]
k_{\mathrm{on},A},\ k_{\mathrm{off},A} &:\ \text{Binding and unbinding rates for receptor A} \\
k_{\mathrm{on},B},\ k_{\mathrm{off},B} &:\ \text{Binding and unbinding rates for receptor B} \\[6pt]
K_g &:\ \text{Tumour growth rate constant (converted to h}^{-1}\text{)} \\
K_k &:\ \text{Maximum tumour kill rate (converted to h}^{-1}\text{)} \\
T_{\max} &:\ \text{Carrying capacity (maximum tumour volume)} \\
IC_{50} &:\ \text{Half-maximal inhibitory concentration for tumour kill}
\end{align}
\paragraph{Viral dynamics model}
Target cells ($T$), eclipse-phase cells ($E$), productively infected cells ($I$), and free virions ($V$) \citep{Nowak2000}.

\paragraph{Tumour–drug–immune interaction}
We consider tumour cells ($T$), chemotherapeutic payload ($M$), immune effectors ($N$), an inflammatory cytokine ($I$), and a tissue-damage biomarker ($K$). This low-order chemo-immuno system follows classical tumour–immune ODEs with chemotherapy cytotoxicity and immunosuppression terms; cytokines and damage markers use standard turnover/indirect–response kinetics \citep{Sharma1998}.

\subsection{COVID-19 data}

We use publicly available data from the European Centre for Disease Prevention and Control (ECDC), comprising daily reported COVID-19 cases and deaths for European countries \citep{ecdc_covid19_data}. These observations are combined with population statistics to obtain normalized trajectories. To stabilize variance and facilitate learning, we apply a $\log(x + 1)$ transformation to the raw time series.  We split countries into training and holdout sets to evaluate generalization across populations. The training set consists of Austria, Belgium, and France ($n=3$), while the holdout set includes Bulgaria, Croatia, Cyprus, Czechia, Denmark, Estonia, and Finland ($n=7$).

\subsection{Dataset Generation}
\label{app:dataset_gen}

\paragraph{ODE Integration and Trajectories.}
For each dynamical system, we generate ground-truth trajectories by integrating the governing Ordinary Differential Equations (ODEs) using a deterministic fourth-order Runge--Kutta (RK4) solver. We use a fixed integration step size $\Delta t$ specific to each system's timescale (see Table~\ref{tab:dataset_master}). 

To evaluate robustness to model mismatch, we sample a unique parameter set for each dataset by applying multiplicative jitter (typically $5\%$--$10\%$) to the default literature parameters. For each data split (Train, Validation, Test), we sample a fresh initial condition $\mathbf{x}(t_0)$ by applying a $10\%$ per-dimension jitter to a nominal starting state, clipping to non-negative values where physically required.

\paragraph{Observation Model and Noise.}
We generate observed data $\mathbf{Y} \in \mathbb{R}^{T \times M}$ from the latent states $\mathbf{x}(t) \in \mathbb{R}^D$ via a linear observation operator $\mathbf{Y}(t) = \mathbf{H}\mathbf{x}(t) + \boldsymbol{\epsilon}$. The matrix $\mathbf{H}$ simulates heterogeneous sensor channels, including direct observations of state subsets, total sums (e.g., total population), or linear combinations. 

We evaluate three noise regimes to test solver robustness:
\begin{itemize}
    \item \textbf{Gaussian Noise:} $\epsilon \sim \mathcal{N}(0, \sigma^2)$.
    \item \textbf{Student-$t$ Noise:} $\epsilon \sim t_\nu$ with degrees of freedom $\nu=5$ (heavy-tailed).
    \item \textbf{Correlated Noise:} AR(1) processes with correlation coefficient $\alpha=0.8$.
\end{itemize}
The noise scale $\sigma$ is set to $5\%$ of the mean absolute amplitude of the state variable ($\sigma = 0.05 \times \text{mean}|\mathbf{x}|$). Full-state snapshots are provided at evenly spaced indices, with the count scaling as $\approx 1.5\sqrt{T}$ to mimic sparse measurement settings.

\paragraph{Dataset Specifications.}
Table~\ref{tab:dataset_master} details the grid dimensions and state variables for all systems evaluated.

\begin{table}[h]
  \centering
  \small
  \caption{Dataset specifications. $D$: state dimension; $\Delta t$: integration step; $N$: number of time points per split. All trajectories start at $t_0=0$.}
  \label{tab:dataset_master}
  \resizebox{\textwidth}{!}{%
  \begin{tabular}{lrrrrrl}
    \toprule
    \textbf{Dataset} & $D$ & $\Delta t$ & $N_{\mathrm{train}}$ & $N_{\mathrm{val/test}}$ & $t_{\mathrm{train}}^{\max}$ & \textbf{State Variables} \\
    \midrule
    \multicolumn{7}{l}{\textit{Population \& Ecological Dynamics}} \\
    SEIR & 4 & 0.2 & 500 & 200 & 99.8 & Susceptible, Exposed, Infected, Recovered \\
    SEIRH & 5 & 0.2 & 500 & 200 & 99.8 & SEIR + Hospitalized \\
    Viral & 4 & 0.2 & 500 & 200 & 99.8 & Target cells, Exposed, Infected, Virus \\
    Consumeristic & 4 & 0.2 & 500 & 200 & 99.8 & $x, y, z, w$ (Social dynamics) \\
    \midrule
    \multicolumn{7}{l}{\textit{Systems Biology \& Pharmacology (PK/PD)}} \\
    Neutrophil & 6 & 0.2 & 500 & 200 & 99.8 & Proliferating, Transit$_{1\text{--}3}$, Reserve, Circulating \\
    Colorectal & 7 & 0.2 & 500 & 200 & 99.8 & Tumor, Necrotic, Lymph, Cells, Macrophage, $I$, $A$ \\
    Tumor & 6 & 0.2 & 500 & 200 & 99.8 & PK/PD with binding \& tumor volume \\
    TMDD Lite & 5 & 0.2 & 500 & 200 & 99.8 & Drug, Receptor, Complex, Production, Internalization \\
    Tumor-Drug-Imm & 5 & 0.5 & 500 & 200 & 249.5 & Tumor, Macrophage, NK cells, IL-2, Kill rate \\
    \bottomrule
  \end{tabular}
  }
\end{table}

\subsection{MAAT Implementation Details}
\label{app:maat_impl}

\paragraph{Kernel State Recovery (KSR).}
We model each state dimension $d \in \{1,\dots,D\}$ using a Gaussian Radial Basis Function (RBF) kernel:
\[
\kappa_d(t,t')=\exp\!\Bigl(-\tfrac{(t-t')^2}{2\sigma_d^2}\Bigr).
\]
The recovered trajectory $\hat{x}_d(t)$ and its time-derivative $\widehat{\dot{x}}_d(t)$ are expanded over the grid points $t_j \in \mathcal{T}$ as:
\[
\hat x_d(t) = \sum_{j=1}^{T}U_{j,d}\,\kappa_d(t,t_j), \qquad
\widehat{\dot x}_d(t) = \sum_{j=1}^{T}U_{j,d}\,\partial_t\kappa_d(t,t_j).
\]

\paragraph{Loss Function and Optimization.}
We optimize the coefficient matrix $U \in \mathbb{R}^{T \times D}$ by minimizing a composite loss function $\mathcal{L}(U)$:
\begin{align}
\mathcal L(U) = 
&\; w_s\,\mathbb E_{(t_k,\mathbf s_k)\in\mathcal M}\!\bigl[\|\hat{\mathbf x}(t_k)-\mathbf s_k\|_2^2\bigr] 
+ w_i\,\mathbb E_{t\in\mathcal T}\!\bigl[\|\mathbf H\hat{\mathbf x}(t)-\mathbf y(t)\|_2^2\bigr] \nonumber \\
&+\gamma\,\mathbb E_{t\in\mathcal T}\!\bigl[\|\widehat{\dot{\mathbf x}}(t)-\mathbf f_0(\hat{\mathbf x}(t))\|_2^2\bigr] 
+\lambda\,\|U\|_F^2
+ w_+\,\|\min(\hat{\mathbf x}, 0)\|_2^2.
\label{eq:ksr_loss}
\end{align}
We use fixed weights:
  \begin{itemize}
      \item \textbf{Weights:} $w_s = w_i = 1.0$, $
  \gamma = 10^{-3}$, $\lambda = 10^{-6}$.
  \end{itemize}

\paragraph{Training Protocol.}
Training proceeds in two phases using the Adam optimizer ($\beta_1=0.99, \beta_2=0.999, \epsilon=10^{-8}$):
\begin{enumerate}
    \item \textbf{Length-scale Selection:} Kernel length-scales are set from the snapshot time
  variance, i.e., ($\sigma_d=\sqrt{\mathrm{Var}
  (\mathcal T)}$), and kept fixed during training (no
  length-scale sweep/tuning in these runs).
    \item \textbf{Optimization:} We optimize $U$ with learning rate $1.0$ for up to $20,000$ iterations, employing early stopping with a patience of $2,000$ steps based on validation loss.
\end{enumerate}
Implementation is in JAX, utilizing JIT compilation for efficient batched evaluation.

\subsection{Baseline Configurations}
\label{app:baselines}

All baselines use fixed hyperparameters across datasets to ensure fair comparison.

\paragraph{Smoothing \& Interpolation Baselines.}
\begin{itemize}
    \item \textbf{Cubic Spline:} Natural cubic splines (SciPy) with analytic derivatives.
    \item \textbf{RBF Interpolant:} Multiquadric RBFs with shape parameter $\varepsilon$ selected via grid search over $\{0.25, \dots, 4.0\} \times (1/\sqrt{\mathrm{Var}(\mathcal{T})})$. Derivatives via central differences ($\Delta = 10^{-3}\Delta t$).
    \item \textbf{Savitzky--Golay:} Window length 25, polynomial order 3, with boundary-aware padding. Analytic filter derivatives are used.
    \item \textbf{TV-Reg Diff:} Total Variation Regularized Differentiation with regularization $\alpha=0.01$, implemented via a Savitzky--Golay proxy (window 21, order 3) for computational efficiency on large grids.
\end{itemize}

\paragraph{Probabilistic \& Neural Baselines.}
\begin{itemize}
    \item \textbf{Gaussian Process (GP):} Independent RBF kernel GPs per dimension. Length-scales optimized via log-marginal likelihood maximization (3 random restarts) searching over factors $\{0.25, \dots, 4.0\} \times \text{std}(\mathcal{T})$.
    \item \textbf{Kalman Smoother:} Constant-velocity kinematic model ($q=1.0, r=0.1$). Derivatives extracted from the smoothed velocity states.
    \item \textbf{Neural ODE:} MLP vector field (width 64, depth 3, tanh activation). Integrated with Dopri5 ($dt_0=0.1$). Trained for 2,000 steps (Adam, lr=$10^{-3}$).
\end{itemize}

\subsection{Symbolic Regression Back-Ends}
Downstream equation discovery is performed on the trajectories recovered by \texttt{MAAT} or baselines.
\begin{itemize}
    \item \textbf{PySINDy:} Sparse identification with a polynomial library (degree $\le 2$). Optimizer: STLS with threshold $0.1$ and decay $0.9$.
    \item \textbf{PySR:} Evolutionary search with 20 iterations, population 1,000. Allowed operators: $\{+, \times\}$. Selection criterion: Validation derivative MSE.
\end{itemize}

\subsection{Computing Infrastructure}
All experiments were conducted on a
  workstation with dual AMD EPYC 7713
  CPUs (2 × 64 cores; 128 physical
  cores total) and an NVIDIA RTX 6000
  Ada Generation GPU (49,140 MiB VRAM,
  ~49 GB).

\end{document}